\documentclass[11pt, letterpaper, logo, onecolumn, copyright]{template_from_google}

\usepackage[T1]{fontenc}

\usepackage{bm} 
\usepackage{nicefrac} 

\usepackage{booktabs} 
\usepackage{enumitem} 
\usepackage{fancyhdr}
\usepackage{multirow} 
\usepackage{tabularx}
\usepackage{multicol}
\usepackage{array}
\usepackage{longtable}
\usepackage{booktabs}
\usepackage{multirow}
\usepackage{adjustbox}
\usepackage{booktabs}

\usepackage{wrapfig} 
\usepackage{tocloft} 
\usepackage{fancybox}
\usepackage{authblk} 
\usepackage[bottom]{footmisc} 

\usepackage{algorithm}
\usepackage[noend]{algpseudocode}

\usepackage{graphicx}
\usepackage{caption}
\usepackage{subcaption}
\providecommand{\safeincludegraphics}[2][]{\IfFileExists{#2}{\includegraphics[#1]{#2}}{\fbox{\parbox{0.9\linewidth}{\centering Missing figure: \texttt{\detokenize{#2}}}}}}

\usepackage[dvipsnames]{xcolor}

\usepackage{xspace} 

\usepackage[authoryear, sort&compress, round]{natbib}

\usepackage[dvipsnames]{xcolor} 

\usepackage{hyperref}

\definecolor{darkblue}{rgb}{0.0, 0.0, 0.6}
\definecolor{darkred}{rgb}{0.7, 0.0, 0.0}
\hypersetup{
  pdffitwindow=true,
  pdfstartview={FitH},
  pdfnewwindow=true,
  colorlinks,
  linktocpage=true,
  linkcolor=darkred,
  urlcolor=darkblue,
  citecolor=darkblue
}

\usepackage{amsmath,amsfonts,bm}

\def\eqref#1{equation~\ref{#1}}

\def\1{\bm{1}}

\DeclareMathAlphabet{\mathsfit}{\encodingdefault}{\sfdefault}{m}{sl}
\SetMathAlphabet{\mathsfit}{bold}{\encodingdefault}{\sfdefault}{bx}{n}

\def\gA{{\mathcal{A}}}

\newcommand{\E}{\mathbb{E}}

\usepackage{amsmath}
\usepackage{amssymb}
\usepackage{mathtools}
\usepackage{bbm}
\usepackage{amsthm}

\usepackage{graphicx}
\usepackage{enumitem}
\usepackage{tikz}
\usepackage{booktabs}
\usepackage{adjustbox}

\usepackage{longfbox}
\usepackage{algorithm}
\usepackage[noend]{algpseudocode}

\usepackage{caption}
\usepackage{subcaption}

\usepackage{xspace}
\usepackage{wrapfig}
\usepackage{multicol,multirow}

\usepackage{tocloft}

\definecolor{Yellow}{RGB}{255, 191, 0}
\definecolor{Green}{RGB}{34, 139, 34}

\usepackage{xcolor}
\definecolor{darkgreen}{RGB}{1, 50, 32} 

\newcommand{\method}{\texttt{PPT}\xspace}

\theoremstyle{plain}
\newtheorem{theorem}{Theorem}[section]

\newtheorem{lemma}[theorem]{Lemma}
\newtheorem{corollary}[theorem]{Corollary}

\newtheorem{remark}[theorem]{Remark}

\newtheorem{proposition}{Proposition}

\definecolor{pptrow}{RGB}{225,236,250} 

\let\cite\citep

\title{\papertitle}
\renewcommand{\headrulewidth}{0.4pt}

\reportnumber{} 

\renewcommand{\today}{}

\title{Explore Broadly, Reason Sharply: Push Small Models toward the Frontier via Sampling}

\author[1,*]{Panagiotis Theodoropoulos}
\author[2]{Nan Jiang}
\author[3]{Xintong Duan}
\author[3]{Ali Hasan}
\author[3]{Yuriy Nevmyvaka}
\author[1]{Evangelos A. Theodorou}
\author[3,\dag]{Wei Deng}

\affil[1]{Georgia Institute of Technology}
\affil[2]{University of Texas at El Paso}
\affil[3]{ML Research, Morgan Stanley}

\begin{abstract}
Power-sharpened sampling is an inference-time alternative to reinforcement-learning (RL) post-training
for enhancing reasoning in large language models (LLMs). High-probability sequences are amplified 
under the base model without parameter updates or external rewards,
avoiding the costly optimization and jagged generalization of RL.
However, this approach faces a fundamental exploration--exploitation trade-off, as
strong sharpening restricts exploration, trapping samplers in plausible but incorrect reasoning trajectories,
whereas weak sharpening leaves the answer distribution diffuse.
To resolve this trade-off, we introduce \textbf{Parallel Power Tempering (PPT)}, instantiating power-sharpened LLM sampling 
via parallel tempering.
Running multiple \emph{interacting} replicas in parallel at different sharpening levels allows lower-power replicas to explore diverse reasoning trajectories and higher-power chains to further exploit higher-likelihood responses favored by the sharpened target. 
Specifically, we tailor \method{} to inference-time sampling by mitigating a truncation bias, identified in prior power samplers, and investigate effective swap strategies under finite memory and compute budgets.
Extensive experimentation shows that \method{} substantially improves single-chain power-sharpened sampling and outperforms RL-post-trained models, producing higher-quality reasoning traces and even achieving performance comparable to frontier models.
\end{abstract}

\begin{document}

\maketitle
{\renewcommand{\thefootnote}{\fnsymbol{footnote}}%
\footnotetext[1]{Work done during an internship at Morgan Stanley.\quad\textsuperscript{$\dagger$}Correspondence: weideng056@gmail.com}}

\vspace{-.3cm}
\section{Introduction}
\label{sec:intro}

\begin{figure}[b]
  \centering
  \vspace{-.3cm}
  \includegraphics[width=\textwidth]{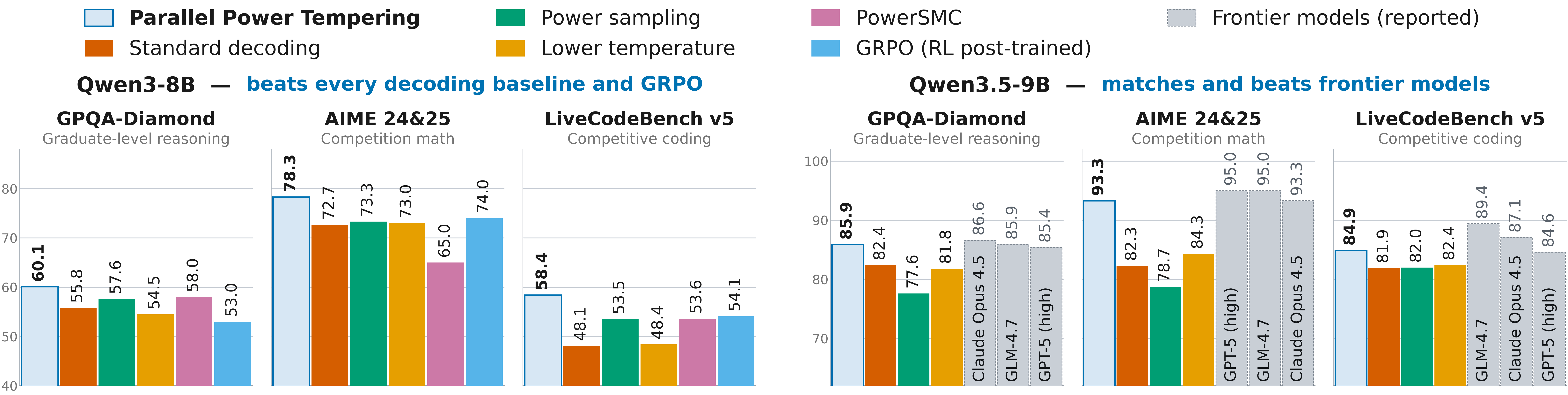}
    \vspace{-.3cm}
  \caption{Single-shot accuracy (\%) of Parallel Power Tempering. 
  Left: Qwen3-8B across three benchmarks. 
  Right: Qwen3.5-9B on GPQA-Diamond, AIME~24\&25 and LiveCodeBench~v5 against decoding baselines 
  and the reported scores of frontier models.
  }
  \label{fig:radar}  \vspace{-0.3cm}
\end{figure}

Reinforcement Learning (RL)-guided post-training with task-specific rewards is a widely used recipe for improving the reasoning capabilities in LLMs
~\citep{shao2024deepseekmath, guo2025deepseekr1}. 
This approach is particularly effective in domains such as mathematics and code generation, where automated verifiers can reliably assess correctness.
However, RL post-training requires access to a verifier or reward model that scores sampled outputs, and it requires expensive gradient-based updates to the model's parameters \citep{shao2024deepseekmath}. Moreover, reliable rewards
are often unavailable for open-ended scientific inquiry, deliberation, and
long-horizon planning.
Even when rewards exist, optimizing for a narrow set of rewarded tasks can erode
prior capabilities and yield \emph{``jagged'' generalization} to nearby
problems~\citep{hu2026breaking}.
These limitations motivate inference-time alternatives that improve reasoning
without parameter updates or external rewards.

Traditional inference-time methods improve generation either by selecting a high-scoring sequence among multiple completions~\citep{huang2025bestofn,NEURIPS2025_1c7eff16}, or by locally reshaping token probabilities with temperature and truncation rules~\citep{holtzman2020curious,meister2023locally,nguyen2024minp,tang2025top}.
The latter are inherently \emph{myopic}, as they modify each decoding step without directly controlling the distribution over complete responses. Neither approach generally samples from a prescribed sequence-level target.

In this vein, recent works suggest that a model's capability can be substantially enhanced at inference time by sampling from a sharpened distribution over the model's outputs~\citep{karan2025reasoning,ji2026scalable}.
Such sequence-level targets naturally depend on Monte Carlo methods for controllable generation and blockwise resampling \citep{mireshghallah2022mix,forristal2023block}.
The cornerstone hypothesis is that reasoning trajectories are latent in models, and thus, sequence-level sharpening amplifies these trajectories, enabling sampling from them without additional RL post-training while yielding comparable inference-time reasoning performance.

However, power-sharpened sampling introduces a fundamental inference-time challenge. 
Sharpening concentrates the sampling distribution around high-likelihood responses, which can improve final-answer quality, but excessive concentration can trap the sampler in locally plausible yet incorrect reasoning paths ~\citep{karan2025reasoning,ji2026scalable}. 
This behavior is especially consequential for multi-step reasoning tasks, where an early mistake can lead to a coherent but wrong solution ~\citep{li2025blinkofaneyetheory}. 
Conversely, flatter distributions explore a broader range of alternative reasoning paths, but may also generate noisy or lower-quality responses that are unsuitable as final outputs ~\citep{troshin2025control}.
Thus, power-sharpened LLM sampling faces a central \textbf{exploration--exploitation trade-off}: the sampler  must explore diverse reasoning trajectories while still exploiting the sharpened distribution to produce high-quality answers.

To address this challenge, we introduce \textbf{Parallel Power Tempering (\method)}, a power-sharpened parallel tempering sampler for LLM inference. 
Our method adapts classical parallel tempering, also known as replica exchange, from multi-modal MCMC ~\citep{swendsen1986replica,geyer1991markov,hukushima1996exchange,earl2005parallel} to sequence-level power-sharpened LLM sampling.
Parallel tempering transfers naturally to the LLM regime. Instead of running a single sampler at one sharpening level, \method{} runs replicas in parallel over a ladder of sharpening levels. Unlike prior LLM replica-exchange methods that vary prefix length \citep{he2026recipes}, all replicas share a common horizon, which enables whole-record swaps.
The lower-sharpening replicas act as exploratory chains that can search broadly across possible responses, while the higher-sharpening replicas act as exploitative chains that concentrate on responses favored by the sharpened objective.
The replicas are coupled through swap moves, which allow neighboring replicas to exchange their generated responses according to a principled and tractable Metropolis--Hastings acceptance rule ~\citep{deng20b,syed2022nonreversible}. 
Every swap decision therefore reduces to the base-model
log-probabilities of the two responses being exchanged, which are already
cached from generation, so a swap requires no additional model evaluations.
In this way, promising reasoning paths discovered by exploratory replicas can be transferred to sharper replicas, while sharper replicas can escape poorly mixed regions by exchanging with flatter chains.
We evaluate \method{} across a range of small open-source models on benchmarks spanning mathematics, coding, and STEM reasoning.
Our main contributions can be summarized as follows:
\begin{itemize}[leftmargin=14pt]
    \item We introduce \method, a novel power-sharpened parallel tempering sampler specially adapted to LLM sampling that runs a ladder
    of replicas at various sharpening levels and couples adjacent replicas
    through swap moves, letting exploratory low-power chains feed diverse states
    to exploitative high-power chains, mitigating the central limitation on the exploration--exploitation trade-off.
    \item We identify a \emph{structural truncation bias} in prior implementations of early-stopped power
          sampling, which parallel tempering cannot repair, and eliminate it via a fixed-horizon construction that preserves the true sharpened target. 
          Since every swap schedule is exact under this construction, we further study efficient swap strategies subject to memory and compute constraints.
    \item Extensive experiments across math, STEM, and code reasoning benchmarks
    on multiple open models show that \method{} substantially outperforms sampling
    and RL baselines in nearly all settings. Notably, with Qwen3.5-9B, it even matches or
    exceeds the performance of frontier models, as illustrated in Figure \ref{fig:radar}.
\end{itemize}

\section{Preliminaries}\label{sec:prelim}

\subsection{Notations} 
Let $\mathbf{x}_0$ denote a given prompt and let
$\mathbf{x}_{0:T} = (\mathbf{x}_0, x_1, \dots, x_T)$ be a sequence of generated tokens appended to that prompt, where each $x_t$ belongs to a finite vocabulary~$\mathcal{V}$.
An autoregressive LLM factorizes the joint distribution over $\mathbf{x}_{0:T}$ via the chain rule $  p_0(\mathbf{x}_{1:T}|\mathbf{x}_0)
  = \prod_{t=1}^{T} p_0(x_t \mid \mathbf{x}_{0:t-1})$.
At fixed horizon $T$, let $\mathcal S_T$ be the finite state space containing sequences of length-$T$ and records terminated at $\tau\leq T$. 
Post-terminal entries are filled with deterministic padding leading to a state $(x_1,\ldots,x_{\tau-1},\mathrm{EOS},\bot,\ldots,\bot)$, 
so $p_0(\mathbf{x}_{1:T}\mid\mathbf{x}_0)$ through $T$ or until EOS.
Sampling proceeds sequentially: at each step $t=1,\dots,T$, a token is drawn from
$x_t \sim p_0(\cdot \mid \mathbf{x}_{0:t-1})$.

\subsection{Power Sharpening for LLMs}
Recent work proposes sampling from a \emph{power-sharpened} variant of this distribution~\citep{karan2025reasoning,ji2026scalable}:
\begin{align} \label{eq:pi-alpha}
  \pi_\alpha(\mathbf{x}_{1:T}\mid \mathbf{x}_0)
  = \frac{p_0(\mathbf{x}_{1:T}\mid \mathbf{x}_0)^{\alpha}}{Z_\alpha(\mathbf{x}_0)},
  \qquad
  Z_\alpha(\mathbf{x}_0)
  =
  \sum_{\mathbf{x}_{1:T}'\in\mathcal{V}^T}
  p_0(\mathbf{x}_{1:T}'\mid \mathbf{x}_0)^{\alpha},
\end{align}
where $\alpha > 1$ is the sharpening parameter and  $Z_{\alpha}(\mathbf{x}_0)$ is the normalizing constant.
Raising the base distribution to a power of $\alpha$ suppresses low-likelihood continuations while amplifying high-likelihood ones~\citep{karan2025reasoning}.
However, direct autoregressive sampling from $\pi_\alpha$  is intractable, since the normalizing constant in Eq. ~\ref{eq:pi-alpha} requires summing the powered likelihood over all possible future completions. 
For this reason Metropolis--Hastings sampling was employed.

\noindent\textbf{Metropolis--Hastings (MH)} is an MCMC algorithm for sampling from a distribution known only up to a normalizing constant~\citep{metropolis1953equation,hastings1970monte}.
Given a target density $\pi_{\alpha}(\mathbf{x})$, MH constructs a Markov chain
$\mathbf{x}^{(0)} \to \mathbf{x}^{(1)} \to \cdots \to \mathbf{x}^{(n)}$
as follows.
From the current state $\mathbf{x}^{(i)}$, a candidate $\mathbf{y}$ is drawn from a proposal distribution. 
Let $g_\alpha$ be the tokenwise-powered proposal
\begin{equation}
\label{eq:proposal}
g_\alpha(x_t\mid\mathbf{x}_0,\mathbf{x}_{<t})
=
\frac{
p_0(x_t\mid\mathbf{x}_0,\mathbf{x}_{<t})^\alpha
}{
\sum_{v\in\mathcal V}
p_0(v\mid\mathbf{x}_0,\mathbf{x}_{<t})^\alpha
}.
\end{equation}
After EOS, its only admissible output is $\bot$, with probability one.
Although tractable, $g_\alpha$ is a tokenwise proposal and is not the
sequence-level target $\pi_{\alpha}$. 
Following \citet{karan2025reasoning}, draw a restart position
$r$ from a state-independent distribution, retain the prefix
$\mathbf{x}_{<r}$, and resample the suffix using
\[
q_r(\mathbf{y}\mid\mathbf{x})
=
\mathbf 1\{\mathbf{y}_{<r}=\mathbf{x}_{<r}\}
\prod_{t=r}^{T}
g_\alpha(y_t\mid\mathbf{x}_0,\mathbf{y}_{<t}).
\]
If $r$ lies after the current EOS, the padding convention makes this a
self-transition. Conditional on the sampled $r$, accept the proposal with
probability
\begin{equation}
\label{eq:accept-ratio}
A_r(\mathbf{x},\mathbf{y})
=
\min \left\{1, 
\frac{
\pi_{\alpha}(\mathbf{y}\mid\mathbf{x}_0)
q_r(\mathbf{x}\mid\mathbf{y})
}{
\pi_{\alpha}(\mathbf{x}\mid\mathbf{x}_0)
q_r(\mathbf{y}\mid\mathbf{x})
}\right\}.
\end{equation}
Each MH restart preserves $\pi_{\alpha}$; thus, the
state-independent mixture over $r$ also preserves $\pi_{\alpha}$.

\medskip

\subsection{Parallel Tempering} 
Parallel tempering (PT) is an MCMC technique designed to improve mixing for target distributions that are difficult to explore directly, such as multimodal or highly concentrated distributions~\citep{swendsen1986replica,earl2005parallel}.
Instead of running one chain, PT runs multiple chains, or \emph{replicas}, in parallel over a ladder of tempered distributions.

PT introduces a sequence of auxiliary distributions $ \pi_1(\mathbf{x}), \pi_2(\mathbf{x}), \dots, \pi_K(\mathbf{x})$.
The joint state of all replicas is denoted by $
  \mathbf{X}
  =
  \left(
    \mathbf{x}^{(1)}, \mathbf{x}^{(2)}, \dots, \mathbf{x}^{(K)}
  \right)$,
where $\mathbf{x}^{(k)}$ is the current state of replica $k$.
The replicas are coupled together through swaps between neighboring replicas $k$ and $k+1$ proposing state exchange as $(\mathbf{x}^{(k)}, \mathbf{x}^{(k+1)}) \to (\mathbf{x}^{(k+1)}, \mathbf{x}^{(k)})$, with acceptance probability
\begin{equation}
  A_{\mathrm{swap}}
  =
  \min\!\left\{
    1,\,
    \frac{
      \pi_k\!\left(\mathbf{x}^{(k+1)}\right)
      \pi_{k+1}\!\left(\mathbf{x}^{(k)}\right)
    }{
      \pi_k\!\left(\mathbf{x}^{(k)}\right)
      \pi_{k+1}\!\left(\mathbf{x}^{(k+1)}\right)
    }
  \right\}.
\end{equation}

\section{Methodology}
\label{sec:method}

\begin{figure}[t]
    \centering
    \vspace{-.3cm}
    \safeincludegraphics[width=.95\linewidth]{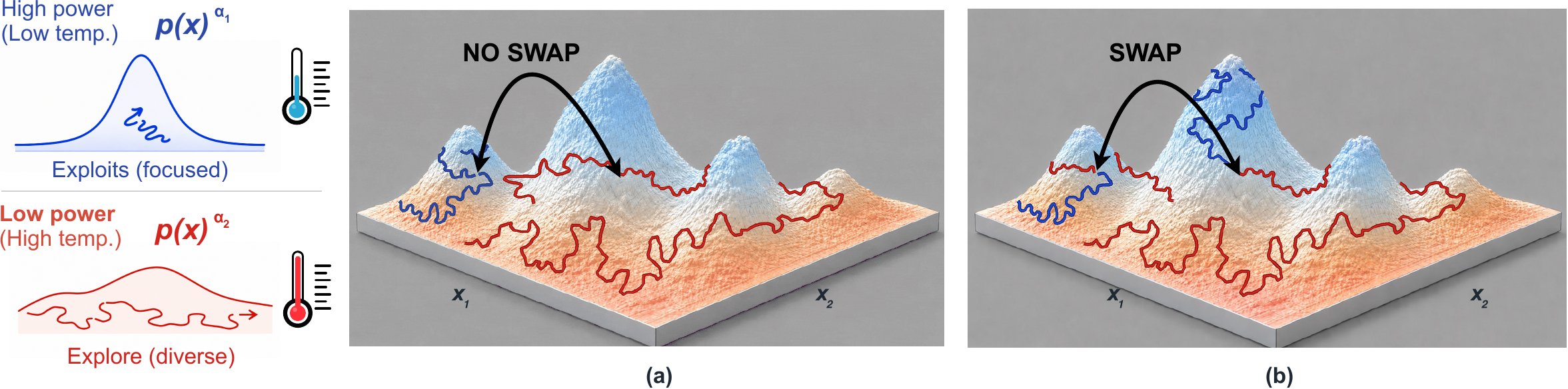}
    \vspace{-.1cm}
    \caption{In our \method, a flat, low-power chain {\color{red}$p(\mathbf{x})^{\alpha_2}$ (red)} \emph{explores freely} while a sharp, high-power chain {\color{blue}$p(\mathbf{x})^{\alpha_1}$ (blue)} exploits local modes to \emph{reason sharply}; periodically the chains propose to swap states. In (a) the swap is rejected and both stay put. In (b), it is accepted, letting the focused chain teleport to a high-probability region the explorer found, rather than climbing over the barrier itself.}
    \label{fig:placeholder}
\end{figure}

\subsection{Motivation: The Exploration Bottleneck}
\label{sec:exploration-horizon}

Recall that at horizon $T$, the sharpened power target has the form
$\pi_\alpha(\mathbf{x})\propto p_0(\mathbf{x})^\alpha$. The power $\alpha$
encodes an exploration--exploitation trade-off that no single value resolves:
raising it concentrates mass on high-likelihood responses---desirable when
high-likelihood traces are more coherent or reliable---but sharpens the
landscape and impedes local exploration, while lowering it flattens the
landscape and eases exploration at the price of leaving good answers diffuse.

To address this exploration--exploitation trade-off, we introduce \textbf{Parallel Power Tempering (\method{})}, which couples multiple replicas through swaps. Specializing the parallel-tempering construction from Section~\ref{sec:prelim} to sequence-level power targets, we choose a ladder of sharpening powers $1\leq\alpha_1<\cdots<\alpha_K$ and assign replica $k$ the target $\pi_{\alpha_k}(\mathbf{x})\propto p_0(\mathbf{x})^{\alpha_k}$ and couple the
replicas through state swaps. Replicas with small $\alpha_k$ act as
high-temperature explorers that range across alternative reasoning paths,
while swaps let the promising trajectories they discover travel up the ladder
to the target chain---the sharpest rung $\alpha_K$---which thus reaches
high-probability regions by exchange rather than by climbing over reasoning
barriers itself. See Figure~\ref{fig:placeholder} for a visual example.

\subsection{Structural Truncation Bias and the Fixed-Horizon Correction}
\label{sec:fixed-horizon}
To apply \method{} to variable-length sequences, refinement must allow
early termination to be revised.
Naïvely extending the early-stopping implementation of \citet{karan2025reasoning} 
to the parallel-tempering setting inherits a one-way truncation bias.
Each suffix proposal regenerates only through the current realized length,
so a shorter candidate $\mathbf{y}$ can satisfy
$q(\mathbf{y}\mid\mathbf{x})>0$ while $q(\mathbf{x}\mid\mathbf{y})=0$.
The exact MH rule should reject this move, but the implementation can accept it
using token scores truncated to the candidate's endpoint. 
Therefore, the accepted shortenings create
uncompensated probability flow toward shorter records, preventing
the sampler from preserving the true target.

Under a uniform accepted-shortening condition, this one-way truncation also yields an asymptotic bias floor.
More formally, let $\mathcal B_h$ denote the records of length at most $h$, and 
let $b_h^{(k)}=\pi_{\alpha_k}(\mathcal B_h^c)$ denote the probability mass
of longer records at rung $k$. 
Assume the MH updates accept a move shortening a record from
above $h$ to at most $h$,  with probability $\delta>0$,
but can never lengthen records.
Then for any initial distribution $\nu$, after $N$ iterations, we show in App. \ref{sec:truncation-lower-bound} that
\begin{equation}
\label{eq:truncation-floor}
\liminf_{N\to\infty}
\operatorname{TV}(\nu_N,\Pi)
\ge 1-\prod_{k=1}^K(1-b_h^{(k)}).
\end{equation}
This inequality quantifies the asymptotic error due to the violation of the target preservation by the 
unbalanced MH rule.
Crucially, this truncation would \emph{ hinder \method's
ability to explore}, as shortened records transferred to lower power rung
would retain their restricted generation budget.


\begin{remark}\label{rem:fixed-hor}
To remove this one-way truncation, we maintain a fixed horizon $T$, through deterministic post-EOS padding $\bot$, 
allowing suffix proposals starting at or before EOS
to regenerate up to $T$ regardless of the current completion length, and explore more effectively.
We show in App.~\ref{sec:theoretical-properties}
that this fixed-horizon correction removes structural truncation bias,
while local MH updates composed with replica swaps preserve the true sharpened target.
\end{remark}

\begin{wrapfigure}[13]{r}{0.3\columnwidth}
  \centering
  \vspace{-.3cm}
  \includegraphics[width=.95\linewidth]{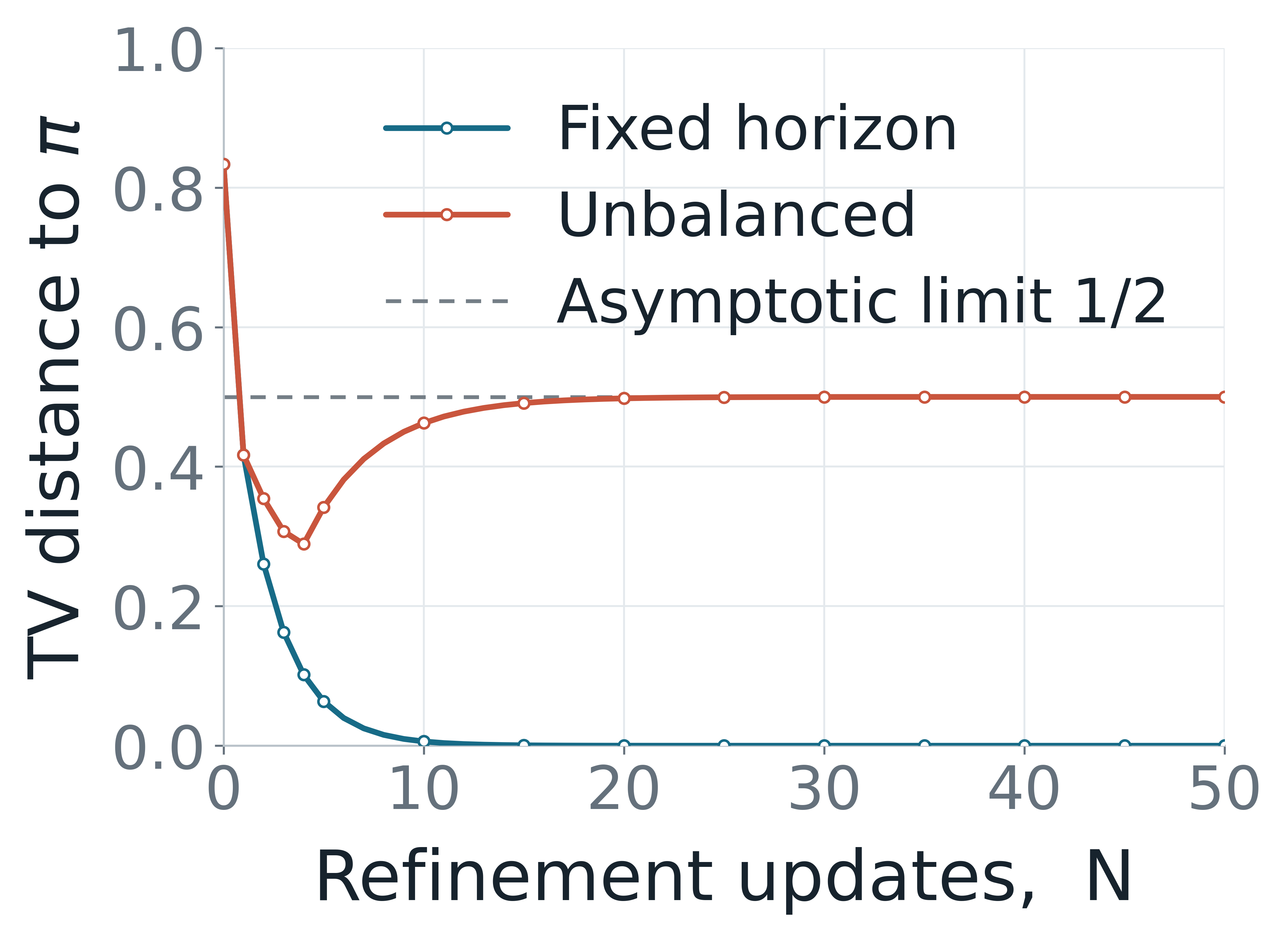}
  \caption{Comparison between fixed-horizon and unbalanced refinement starting from $X_3$.}
  \label{fig:toy-tv-x4}
\end{wrapfigure}
\textbf{Error Floor: Example. }
To better understand the effect of the one-way truncation, consider
one chain at a fixed horizon $T=2$, with power $\alpha=2$ and
vocabulary $\{a,e\}$, where $e$ is the terminal token. The model assigns
probability $1/2$ to each token at every unfinished prefix. On the common record space
$(X_1,X_2,X_3,X_4)=\bigl((e,\bot),(a),(a,e),(a,a)\bigr)$,
and the powered target is $\pi=(\frac23,0,\frac16,\frac16)$, assigning zero target mass
on the unfinished one-token record $X_2$.
Assume, both samplers start from $X_3=(a,e)$ and perform $N$ MH updates.
Once unbalanced refinement accepts a one-token
record, it cannot revisit either two-token continuation. Fixed-horizon
refinement retains this possibility: from $X_1$, it returns to a two-token
record with probability $1/8$ per update. Thus an early EOS remains revisable.
For $N\ge1$, the fixed-horizon and truncating laws $\nu_N$ and
$\widetilde\nu_N$ satisfy
\begin{align*}
\operatorname{TV}\!\left(\nu_N,\pi\right)
&=
\frac23\left(\frac58\right)^N
\xrightarrow[N\to\infty]{}0,
&
\operatorname{TV}\!\left(\widetilde\nu_N,\pi\right)
&\xrightarrow[N\to\infty]{}\frac12,
\end{align*}
Although, unbalanced refinement initially reduces the error, as illustrated in Figure~\ref{fig:toy-tv-x4},
exploration is eventually confined to the one-token records. More details are left for App. \ref{sec:theoretical-properties}.

\subsection{Parallel Power Tempering}
\label{sec:repps-llm}
Following the progressive schedule of \citet{karan2025reasoning}, we choose a
block width $B$ and define
\[
T_m=\min\{mB,T\},\qquad m=0,\ldots,M,\qquad M=\lceil T/B\rceil.
\]
The fixed-horizon construction of Remark \ref{rem:fixed-hor} is applied at every stage. 
More specifically, at stage $m$, all replicas share horizon $T_m$, with positions after EOS
padded by $\bot$. Let us denote the record of replica $k$ by
$\mathbf{x}^{(k)}\in\mathcal S_{T_m}$, targeting 
\(
\pi_k(\mathbf{x})\propto p_0(\mathbf{x})^{\alpha_k},
\)
with stage-$m$ joint target being expressed as
\begin{equation}
\label{eq:stage-joint-target}
    \Pi_m\!\left(\mathbf{x}^{(1)},\ldots,\mathbf{x}^{(K)}\right)
    :=
    \prod_{k=1}^{K}\pi_k\!\left(\mathbf{x}^{(k)}\right).
\end{equation}
Because $\Pi_m$ factorizes, its $k^\text{th}$ marginal is $\pi_k$. Since $T_M=T$,
the designated rung $k^\star$ has marginal
$\pi_{k^\star}=\pi_{\alpha^\star}$, which is exactly the target at the horizon-$T$
(Section~\ref{sec:prelim}). The auxiliary replicas provide potential transport
paths without altering this output marginal; their mixing advantage is
conditional on their local kernels being faster than the output-rung kernel.

\noindent\textbf{(1) Block proposal.}
Starting from its stage-$(m-1)$ record, we extend each open replica by up to
$B$ tokens using the tokenwise-powered proposal $g_{\alpha_k}$ from
Eq.~\ref{eq:proposal}:
\begin{equation}
\label{eq:proposal-g}
    x_t^{(k)}\sim
    g_{\alpha_k}\!\left(\,\cdot\mid\mathbf{x}^{(k)}_{<t}\right),
    \qquad t=T_{m-1}+1,\ldots,T_m.
\end{equation}
Generation stops at EOS, so only non-terminated records require model calls.
Sampling from $g_{\alpha_k}$ corresponds to token-level decoding at temperature
$1/\alpha_k$, which differs from the sequence-level target. Therefore, block
extension initializes stage $m$ but does not draw from target $\Pi_m$.

\noindent\textbf{(2) Local refinement.}
At stage $m$ and rung $k$, fix the current record $\mathbf{x}\in\mathcal S_{T_m}$.
Each local update draws a resampling position
$r\sim\operatorname{Unif}\{1,\ldots,T_m\}$, retains the prefix
$\mathbf{x}_{<r}$, and regenerates the suffix, yielding
$\mathbf{y}\sim q_r(\cdot\mid\mathbf{x})$, where
\begin{equation}
\label{eq:fixed-horizon-proposal}
q_r(\mathbf{y}\mid\mathbf{x})
    :=
    \mathbf 1\{\mathbf{y}_{<r}=\mathbf{x}_{<r}\}
    \prod_{t=r}^{T_m}
    g_{\alpha_k}(y_t\mid\mathbf{y}_{<t}).
\end{equation}
The MH correction accepts the candidate with probability
\begin{equation}
\label{eq:local-accept}
    A^\text{refine}=\min\left\{1,
    \frac{
        \pi_k(\mathbf{y})\,q_r(\mathbf{x}\mid\mathbf{y})
    }{
        \pi_k(\mathbf{x})\,q_r(\mathbf{y}\mid\mathbf{x})
    }\right\}.
\end{equation}
For each fixed $r$, the resulting MH kernel preserves $\pi_k$; hence, so does
the state-independent uniform mixture over $r$. Because $r$ may fall anywhere
in the record, refinement can revise decisions before the newest block.
Forward and reverse probabilities are both evaluated through $T_m$. 

\noindent\textbf{(3) Swaps.}
At horizon $T_m$, to couple the chains along the ladder, we sweep through the
adjacent pairs in order $k=1,\ldots,K-1$ and propose exchanging the records at
rungs $k$ and $k+1$,
\[
(\mathbf{x}^{(k)},\mathbf{x}^{(k+1)})
\mapsto
(\mathbf{x}^{(k+1)},\mathbf{x}^{(k)}),
\]
committing each accepted swap before proceeding to the next pair.
This sequential ordering allows a state to traverse multiple rungs within one
sweep. Related ordered swap schedules are studied in classical parallel tempering
\citep{DBLP:conf/aaai/000200LL23}.
\begin{figure}[t]
    \centering
    \vspace{-.5cm}
    \makebox[\linewidth][c]{%
        \includegraphics[width=\linewidth,trim=0 1cm 0 0,clip]{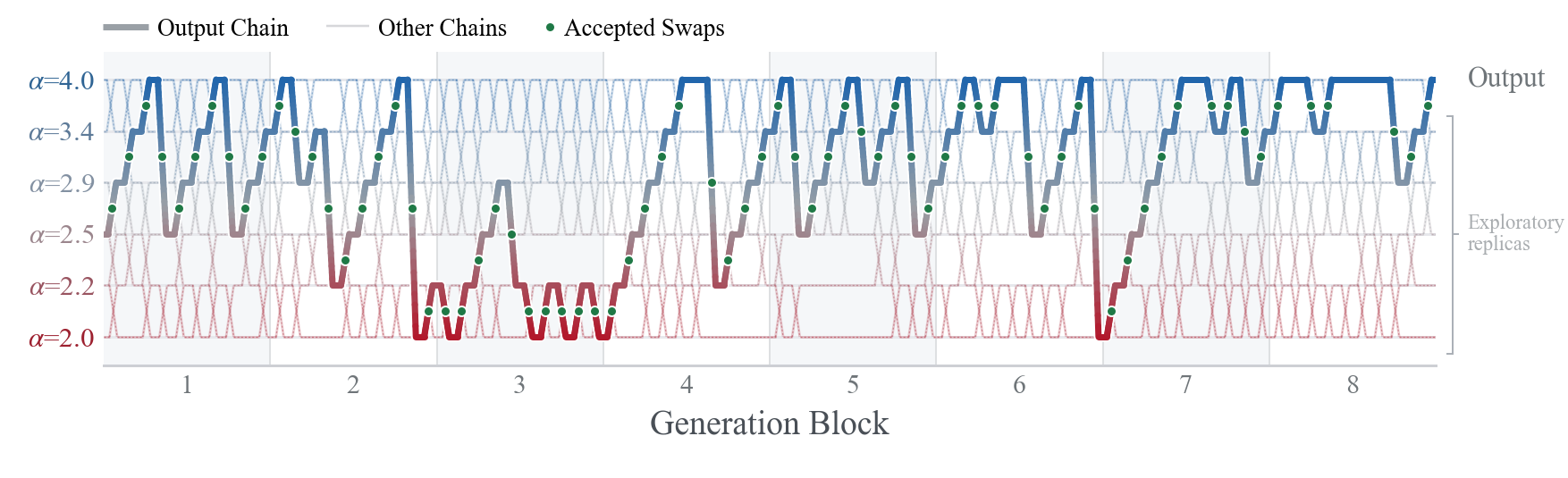}%
    }
    \caption{\textbf{Parallel tempering across the sharpening ladder.}
    Green dots mark accepted adjacent swaps; the colored gradient line traces a sampled trajectory across replicas and generation blocks.}
    \label{fig:parallel-tempering-trace}
\end{figure}

In Figure~\ref{fig:parallel-tempering-trace}, the sequence that is ultimately
returned has been carried by accepted swaps through every rung of the ladder,
from the most exploratory replica to the sharpest, before reaching the output
rung.
The MH acceptance probability is
\begin{equation}
\label{eq:swap-accept-llm}
    A^\text{swap}=\min\!\left\{1,
    \frac{
    \pi_k(\mathbf{x}^{(k+1)})
    \pi_{k+1}(\mathbf{x}^{(k)})
    }{
    \pi_k(\mathbf{x}^{(k)})
    \pi_{k+1}(\mathbf{x}^{(k+1)})
    }\right\}.
\end{equation}
Substituting $\pi_k\propto p_0^{\alpha_k}$ cancels the normalizing constants and gives
\begin{equation}
\label{eq:stage-swap}
    A^\text{swap}=\min\!\left\{1,
    \exp\!\left[
    (\alpha_{k+1}-\alpha_k)
    \left(
    \log p_0(\mathbf{x}^{(k)})
    -\log p_0(\mathbf{x}^{(k+1)})
    \right)
    \right]\right\}.
\end{equation}
\paragraph{Free Exploration.}
Swaps can transport promising trajectories from exploratory rungs to a
sharper rung whose local updates can complete them, potentially \emph{exponentially} accelerating
discovery of better modes \citep{dong2022spectral}, rather than improving only linearly with additional independent samples.
Notably, this exploration comes with minimal computational overhead as
each adjacent exchange is an MH move evaluated from cached sequence
likelihoods, so the exchange itself requires no additional model passes
(Table~\ref{tab:time-memory-main}).
Algorithm~\ref{alg:PPT} specifies the complete schedule, and
Appendix~\ref{apx:implement} gives implementation details; 
after stage $M$, the record at rung $k^\star$ is returned.

\subsection{Ladder Design}
\label{sec:ladder-design}
To effectively capitalize the benefit of Parallel Tempering, adjacent chains overlap sufficiently. 
Powers placed too far apart exchange rarely and cut the ladder into isolated pieces;
powers placed too close spend replicas without sufficient exploration. For
fixed endpoints $\alpha_1,\alpha_K$ and $K$ chains, an optimal ladder design is the
\emph{equi-accepting} ladder~\citep{rathore2005optimal, DBLP:conf/aaai/000200LL23}. For expected
acceptance $\mathcal A_k$ between rungs $k$ and $k+1$, the intermediate powers are chosen such that
\begin{equation}
\label{eq:equi-accepting-ladder}
\gA_1=\gA_2=\cdots=\gA_{K-1}.
\end{equation}
A pair with substantially lower acceptance is a bottleneck that throttles
transport along the whole ladder, while unusually high acceptance signals two
powers that are needlessly close; equalizing acceptances maximizes the weakest
link (Appendix~\ref{subsec:maximin-ladder}).
\begin{remark}
\label{rem:ar_geom_seq}
The acceptance rate $\gA_k$ in Eq.~\ref{eq:stage-swap} is governed by the gap
$\alpha_{k+1}-\alpha_k$ and the spread of the sequence
log-likelihood $\log p_0$.
When this spread
scales as $1/\alpha$, then $\gA_k$
becomes a function of the ratio $\alpha_{k+1}/\alpha_k$, and the
equi-accepting condition of Eq.~\ref{eq:equi-accepting-ladder} reduces
to the geometric ladder
\[
\alpha_k
=
\alpha_1
\left(\frac{\alpha_K}{\alpha_1}\right)^{\frac{k-1}{K-1}},
\qquad k=1,\ldots,K,
\]
In practice, we initialize with a geometric ladder and tune the
intermediate powers empirically until the measured swap acceptances
are approximately equal. More details are left for App. \ref{sec:appendix-ladder-design}.
\end{remark}

\noindent\textbf{Schedule}
Each adjacent-pair swap is itself an MH move on the joint target, so any composition
of swaps preserves $\Pi_m$ at every stage $m$, and hence $\Pi$. Therefore, the communication
schedule cannot change what we sample, only how fast records travel across the
ladder. Choosing one is an efficiency question, and its answer depends on
the regime.
In our implementation, we employ an ordered sweep of adjacent exchanges, 
applying each accepted swap immediately, as presented in Section~\ref{sec:repps-llm}.
This allows a state to traverse several rungs within a single sweep.
An alternative schedule is an asymptotically-optimal swap schedule in classical parallel tempering: the deterministic
even--odd (DEO) communication defined in
\citep{okabe2001replica,syed2022nonreversible}, which attempts a
single parity matching of disjoint interfaces and alternates between the two matchings, so that the swaps within a round can run in parallel.

\begin{remark}
If acceptance probabilities are history-independent and equal to a constant $\bar \gA$, then the expected
iteration counts for a tagged-replica round trip are
\begin{equation}
    \E T_{\mathrm{ADJ}} = K\left(1+(K-1)\frac{1-\bar\gA}{\bar\gA}\right), \qquad \E T_{\mathrm{DEO}} = 2K\left(1+(K-1)\frac{1-\bar\gA}{\bar\gA}\right),
\end{equation}
Thus, ordered ADJ requires half as many expected iterations as DEO,
where an ADJ iteration comprises a complete adjacent sweep and a DEO
iteration comprises one parity layer. This reduction comes at the cost
of executing all $K-1$ swap attempts sequentially. In our regime, where
we use few replicas and communication overhead is small relative to
local refinement, this tradeoff favors ADJ. Conversely, DEO becomes faster for
sufficiently large $K$, since ADJ's sequential communication overhead
grows quadratically with the number of replicas, compared to the DEO's linear scaling.
\end{remark}

\section{Experiments}\label{sec:exp}
\raggedbottom
\setlength{\textfloatsep}{8pt plus 2pt minus 2pt}
\setlength{\floatsep}{8pt plus 2pt minus 2pt}
\setlength{\intextsep}{8pt plus 2pt minus 2pt}

\begin{table*}[t]
  \centering
  \small
  \setlength{\tabcolsep}{5pt}
  \renewcommand{\arraystretch}{0.92}
  \caption{Pass@1 accuracy (\%) with one returned completion per item. \textbf{Bold} marks the best result, including ties, within each model group}
  \label{tab:main-results}
  \begin{adjustbox}{max width=\textwidth}
  \begin{tabular}{l|cccccc}
    \toprule
    \textbf{Method} & \textbf{MATH500} & \textbf{GPQA} & \textbf{HumanEval} & \textbf{GSM8K} & \textbf{AIME 24\&25} & \textbf{LCB v5} \\
    \midrule
    \multicolumn{7}{l}{\textbf{Qwen3-4B}} \\
    \hspace{12pt} Standard
    & 83.3 & 51.0 & 85.4 & 94.6 & 70.7 & 45.6 \\
    \hspace{12pt} Lower Temperature
    & 85.1 & 48.5 & 86.6 & 94.5 & 72.0 & 44.8 \\
    \hspace{12pt} Power Sampling
    & 83.5 & 51.0 & 90.2 & 94.2 & 73.3 & 55.1 \\
    \hspace{12pt} PowerSMC
    & 83.2 & 47.5 & 87.9 & 94.3 & 61.7 & 52.3 \\
    \hspace{12pt} GRPO
    & 85.6 & 50.0 & 91.6 & 94.6 & 73.3 & 52.7 \\
    \rowcolor{pptrow} \hspace{12pt} \textbf{\method{} (Ours)}
    & \textbf{86.0} & \textbf{53.6} & \textbf{93.9} & \textbf{94.8} & \textbf{78.3} & \textbf{57.9} \\
    \midrule
    \multicolumn{7}{l}{\textbf{Qwen3-8B}} \\
    \hspace{12pt} Standard
    & 83.9 & 55.8 & 84.9 & 94.6 & 72.7 & 48.1 \\
    \hspace{12pt} Lower Temperature
    & 84.6 & 54.5 & 87.2 & 95.0 & 73.0 & 48.4 \\
    \hspace{12pt} Power Sampling
    & 82.2 & 57.6 & 90.7 & 95.2 & 73.3 & 53.5 \\
    \hspace{12pt} PowerSMC
    & 84.1 & 58.0 & 90.2 & \textbf{95.5} & 65.0 & 53.6 \\
    \hspace{12pt} GRPO
    & 86.4 & 53.0 & 93.3 & 94.6 & 74.0 & 54.1 \\
    \rowcolor{pptrow} \hspace{12pt} \textbf{\method{} (Ours)}
    & \textbf{88.0} & \textbf{60.1} & \textbf{94.9} & \textbf{95.5} & \textbf{78.3} & \textbf{58.4} \\
    \bottomrule
  \end{tabular}
  \end{adjustbox}
\end{table*}

\subsection{Main Results}
We compare \method{} with standard and low-temperature decoding, Power Sampling~\citep{karan2025reasoning}, PowerSMC~\citep{DBLP:arxiv/power-smc}, and GRPO~\citep{shao2024deepseekmath}, which is a training-based reference.
Table~\ref{tab:main-results} reports results for Qwen3-4B and Qwen3-8B on MATH500~\citep{lightman2024lets}, GPQA~\citep{rein2024gpqa}, HumanEval~\citep{chen2021evaluatingllmcode}, GSM8K~\citep{gsm8k_dataset}, AIME 24\&25 ~\citep{aime24,aime25}, and LiveCodeBench (LCB) v5~\citep{DBLP:conf/iclr/JainHGLYZWSSS25}, with a common benchmark-specific completion cap shared by all methods. Table~\ref{tab:qwen35-9b} extends the comparison to the more recent and capable Qwen3.5-9B on the three most challenging benchmarks: GPQA, AIME 24\&25, and LCB v5.

Across both tables, \method{} attains the best or tied-best accuracy in all $15$ model--benchmark settings, spanning mathematics, science, and code; its only tie is on the saturated GSM8K.
It is demonstrated that single-chain sharpening is unreliable on strong models, as Power Sampling falls below standard decoding in most Qwen3.5-9B benchmarks, PowerSMC trails standard decoding by a wide margin on AIME 24\&25, and lower-temperature decoding yields only modest gains.
Conversely, \method{} targets the same powered distribution and outperforms every sampling-based baseline, never regresses below standard decoding, and even achieves higher accuracy than GRPO without additional parameter updates or rewards.
Notably, using Qwen3.5-9B, we observed performance that matches and even exceeds frontier models. 
Additional results are in Appendix~\ref{apx:extra-exp}.

\textbf{Runtime and Memory Cost  }
\label{sec:runtime-memory}
Furthermore, we provide a breakdown of the computational cost of \method{} in the LCB v5 dataset.
Table~\ref{tab:time-memory-main} reports per problem: peak-memory usage, local-update time,
and replica-exchange overhead for the single chain ($K=1$) and \method{} ($K=4$ for Qwen3-8B, $K=3$ for Qwen3.5-9B).
The efficiency of our implementation is highlighted as batching local 
updates across replicas accelerates wall-clock inference-time.
As shown in Table~\ref{tab:time-memory-main}, the multi-replica runs require
only $1.03\times$ and $1.72\times$ the single-chain local-update
wall-clock time for Qwen3-8B and Qwen3.5-9B, respectively, while 
replica exchange accounts for only $0.1\%$ of the total time in both cases.
However, this efficiency comes at the expense of higher peak memory usage: approximately $3.73\times$ and $2.47\times$ higher than the single-chain values.



\subsection{Compute-matched Controls}
\label{sec:compute-matched}

\begin{table*}[!t]
\centering
\vspace{-10pt}
\begin{minipage}[t]{0.44\textwidth}
\vspace{0pt}
\centering
\captionsetup{font=normalsize,skip=4pt}
    \caption{Pass@1 Accuracy using Qwen3.5-9B on GPQA, AIME 24\&25 and LCB v5.}
    \label{tab:qwen35-9b}
    \setlength{\tabcolsep}{3pt}
\begin{adjustbox}{width=\textwidth}
\begin{tabular}{lccc}
\toprule
\textbf{Qwen3.5-9B} & \textbf{GPQA} & \textbf{AIME 24\&25} & \textbf{LCB} \\
\midrule
Standard          & 82.4 & 82.3 & 81.9 \\
Power Sampling    & 77.6 & 78.7 & 82.0 \\
Lower Temperature & 81.8 & 84.3 & 82.4 \\
\midrule\textcolor{gray}{GPT-5$^{\dagger}$} & \textcolor{gray}{85.4} & \textcolor{gray}{95.0} & \textcolor{gray}{84.6} \\\textcolor{gray}{Opus 4.5$^{\dagger}$} & \textcolor{gray}{86.6} & \textcolor{gray}{93.3} & \textcolor{gray}{87.1} \\\textcolor{gray}{GLM 4.7$^{\dagger}$} & \textcolor{gray}{85.9} & \textcolor{gray}{95.0} & \textcolor{gray}{89.4} \\
\midrule
\rowcolor{pptrow} \textbf{\method{}(Ours)} & \textbf{85.9} & \textbf{93.3} & \textbf{84.9} \\
\bottomrule
\end{tabular}
\end{adjustbox}
  \par\vspace{2pt}{\raggedright\footnotesize\color{gray}\noindent$^{\dagger}$ Reported values.\par}
  \end{minipage}\hfill
\begin{minipage}[t]{0.54\textwidth}
\vspace{0pt}
    \centering
    \captionsetup{font=normalsize,skip=4pt,width=\linewidth}
    \caption{Per-example runtime and peak memory for our implementation.}
    \label{tab:time-memory-main}
    \small
    \setlength{\tabcolsep}{1pt}
    \setlength{\aboverulesep}{0.4ex}
    \setlength{\belowrulesep}{0.55ex}
    \renewcommand{\arraystretch}{0.97}
    \begin{tabular*}{\linewidth}{@{}l@{\extracolsep{\fill}}ccc@{}}
    \toprule
    \textbf{Method} &
    \shortstack{\textbf{Peak KV cache}\\ \textbf{memory (GiB)}} &
    \shortstack{\textbf{Local-update}\\ \textbf{time (sec)}} &
    \shortstack{\textbf{Swap}\\ \textbf{overhead (sec)}} \\
    \midrule
    \multicolumn{4}{@{}l}{\textbf{Qwen3-8B}} \\
    Single chain
        & 0.86 & 36.1 & -- \\
    \method{} {\tiny ($K=4$)}
        &3.21 & 37.2 & 0.037 \\
    \midrule
    \multicolumn{4}{@{}l}{\textbf{Qwen3.5-9B}} \\
    Single chain
        &3.06 & 43.8 & -- \\
    \method{} {\tiny ($K=3$)}
        &7.56 & 75.4 & 0.08 \\
    \bottomrule
    \end{tabular*}
\end{minipage}
\end{table*}

\leavevmode\vspace*{-\dimexpr\baselineskip+\parskip\relax}\par 
\begin{wrapfigure}{r}{0.42\textwidth}
\vspace{-.2cm}
\begin{minipage}{\linewidth}
    \centering
    \captionsetup{font=normalsize,skip=4pt}
    \captionof{table}{Compute-matched Qwen3-8B accuracy (\%) between our multi-chain \method $(K>1)$ and single chain $(K=1)$ samplers.}
    \label{tab:compute-matched-main}
    \setlength{\tabcolsep}{3pt}
    \begin{tabular}{lcc}
      \toprule
      \textbf{Benchmark} & \textbf{\method{}} & \textbf{Single-chain}\\
      \midrule
      MATH500 & \textbf{88.0} & 85.6 \\
      GPQA & \textbf{60.1} & 59.1 \\
      \small  AIME 24\&25 & \textbf{78.3} & 75.3 \\
      LCB v5 & \textbf{58.4} & 54.9 \\
      \bottomrule
    \end{tabular}
\end{minipage}
\vspace{-9pt} 
\end{wrapfigure}

\noindent\textbf{Additional single-chain refinement.}
To test whether additional refinement of a single chain can recover the gains of \method{}, we increase the number of MCMC steps in single chain settings ($K=1$) to match the total compute in terms of decode tokens generated by our multi-chain \method{} on Qwen3-8B.
Table~\ref{tab:compute-matched-main} shows that \method{} remains consistently more accurate across all four benchmarks.
The largest gaps occur on the hardest tasks, namely LCB v5 and AIME 24\&25, where \method{} reaches $58.4\%$ and $78.3\%$ compared with $54.9\%$ and $75.3\%$ for the extended single chain, respectively.
Thus, additional single-chain refinement does not close the performance gap at the matched compute budget.
\par
\par

\noindent\textbf{$\mathbf{K}$ uncoupled chains.}
To isolate the contribution of parallel tempering, Table~\ref{tab:swap-ablation} compares \method{} with an uncoupled ladder on both Qwen3 models. The uncoupled ladder runs the same number of chains, each refined independently over a fixed horizon with no swaps.
We report two metrics: i) \emph{Pass@1} scores a single trace; the coldest rung for \method{}, and the rung with the highest accumulated log-likelihood for the uncoupled ladder, and ii) \emph{Vote} takes the majority answer across all traces.
Table \ref{tab:swap-ablation} shows that enabling swaps consistently improves both the Pass@1 and voting accuracy in every model--benchmark setting.
Overall, we can see that significant gains are attributed to the exchange mechanism, not merely running and aggregating more replicas.

\begin{table*}[!t]
\centering
\caption{Compute-matched accuracy (\%) with and without swaps using the same number of chains. Pass@1 scores a single selected trace; Vote takes the majority answer across traces. \textbf{Bold} marks the better result.}
\label{tab:swap-ablation}
\setlength{\tabcolsep}{5pt}
\begin{tabular}{l|cc|cc|cc}
\toprule
\multirow{2}{*}{\textbf{Method}} &
\multicolumn{2}{c|}{\textbf{MATH 500}} &
\multicolumn{2}{c|}{\textbf{AIME 24\&25}} &
\multicolumn{2}{c}{\textbf{GPQA}} \\
\cmidrule(lr){2-3} \cmidrule(lr){4-5} \cmidrule(lr){6-7}
& Pass@1 & Vote
& Pass@1 & Vote
& Pass@1 & Vote \\
\midrule
\multicolumn{7}{l}{\textbf{Qwen3-4B}} \\
\hspace{12pt} Uncoupled ladder
  & 84.2 & 87.6
  & 75.7 & 78.3
  & 52.3 & 48.5 \\
\rowcolor{pptrow} \hspace{12pt} \textbf{\method{} (Ours)}
  & \textbf{86.0} & \textbf{88.8}
  & \textbf{78.3} & \textbf{80.0}
  & \textbf{53.6} & \textbf{55.1} \\
\midrule
\multicolumn{7}{l}{\textbf{Qwen3-8B}} \\
\hspace{12pt} Uncoupled ladder
  & 86.0 & 87.4
  & 76.7 & 80.0
  & 57.6 & 58.9 \\
\rowcolor{pptrow} \hspace{12pt} \textbf{\method{} (Ours)}
  & \textbf{88.0} & \textbf{90.0}
  & \textbf{78.3} & \textbf{81.7}
  & \textbf{60.1} & \textbf{61.6} \\
\bottomrule
\end{tabular}
\end{table*}

\subsection{Ablation Studies}
\label{sec:ablation}

\noindent\textbf{Number of replicas.}
We sweep $K=1$--$5$ with fixed sharpening range, block size, and MCMC steps, using our fixed-horizon kernel at $K=1$ (Figure~\ref{fig:replica_tradeoff}, left).
Compute is measured in wall-clock time in seconds per example.
Parallel batching lowers the cost per replica, so running $K$ chains
scales more favorably than $K$-fold times a single-chain run. 
For example, with $K=4$,
the total cost is only $2.1\times$ and $1.2\times$ that of a single chain
on MATH500 and HumanEval, respectively.
These results further demonstrate the favorable runtime scaling of \method{} as the number of chains increases.
Beyond five replicas, per-replica accuracy plateaus, with diminishing returns for further increasing. 
We found that accuracy saturates around four to five replicas: $K=4$, gains $2.7$ and $3.9$ percentage 
points over $K=1$, while $K=5$ peaks at $+3.2$ and $+4.5$.

\begin{figure}[t]
    \centering
    \begin{minipage}[t]{0.5\linewidth}
        \vspace{10pt}
        \includegraphics[width=\linewidth]{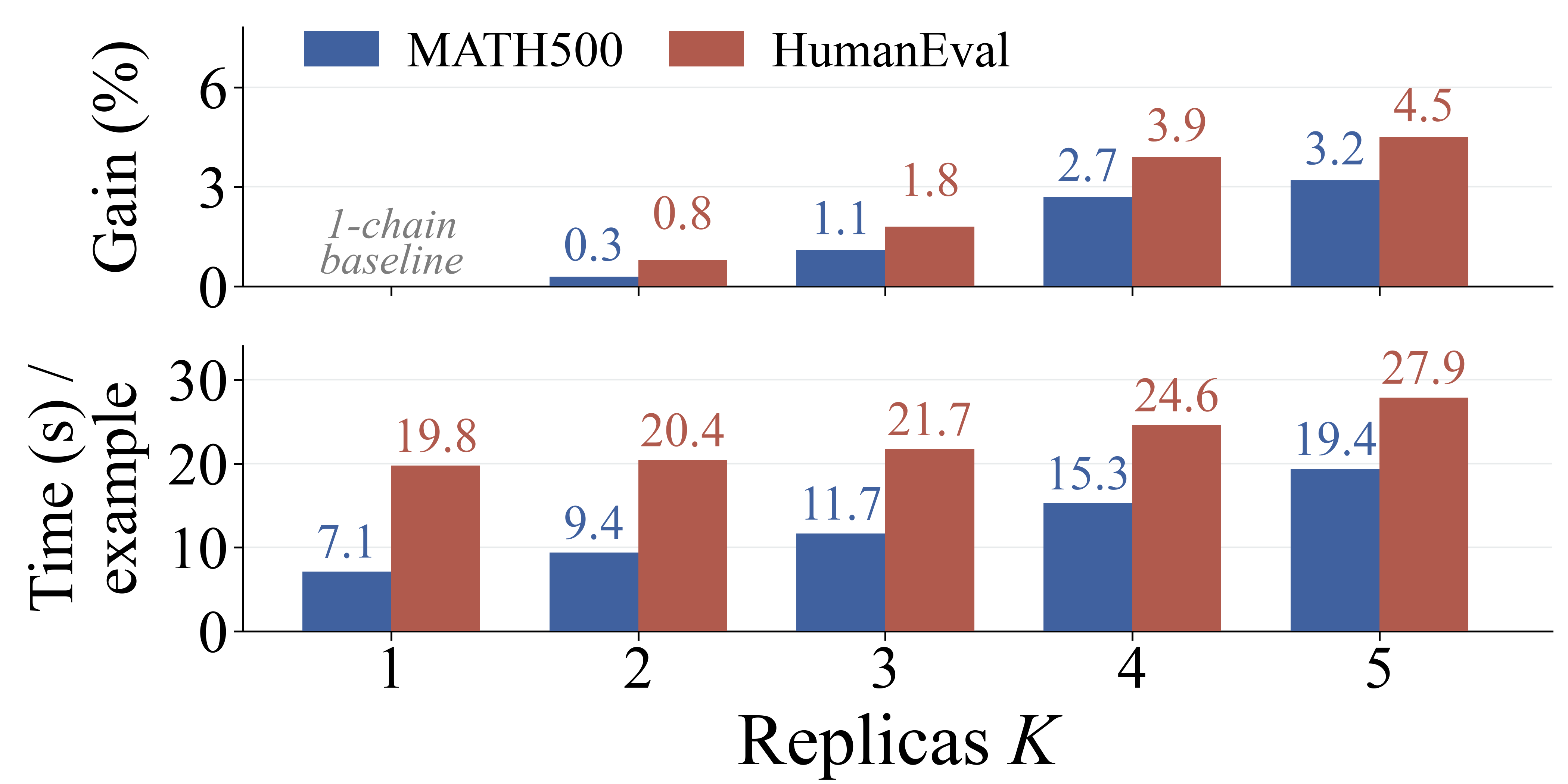}
    \end{minipage}%
    \hfill
    \begin{minipage}[t]{0.48\linewidth}
        \vspace{10pt}
        \includegraphics[width=\linewidth]{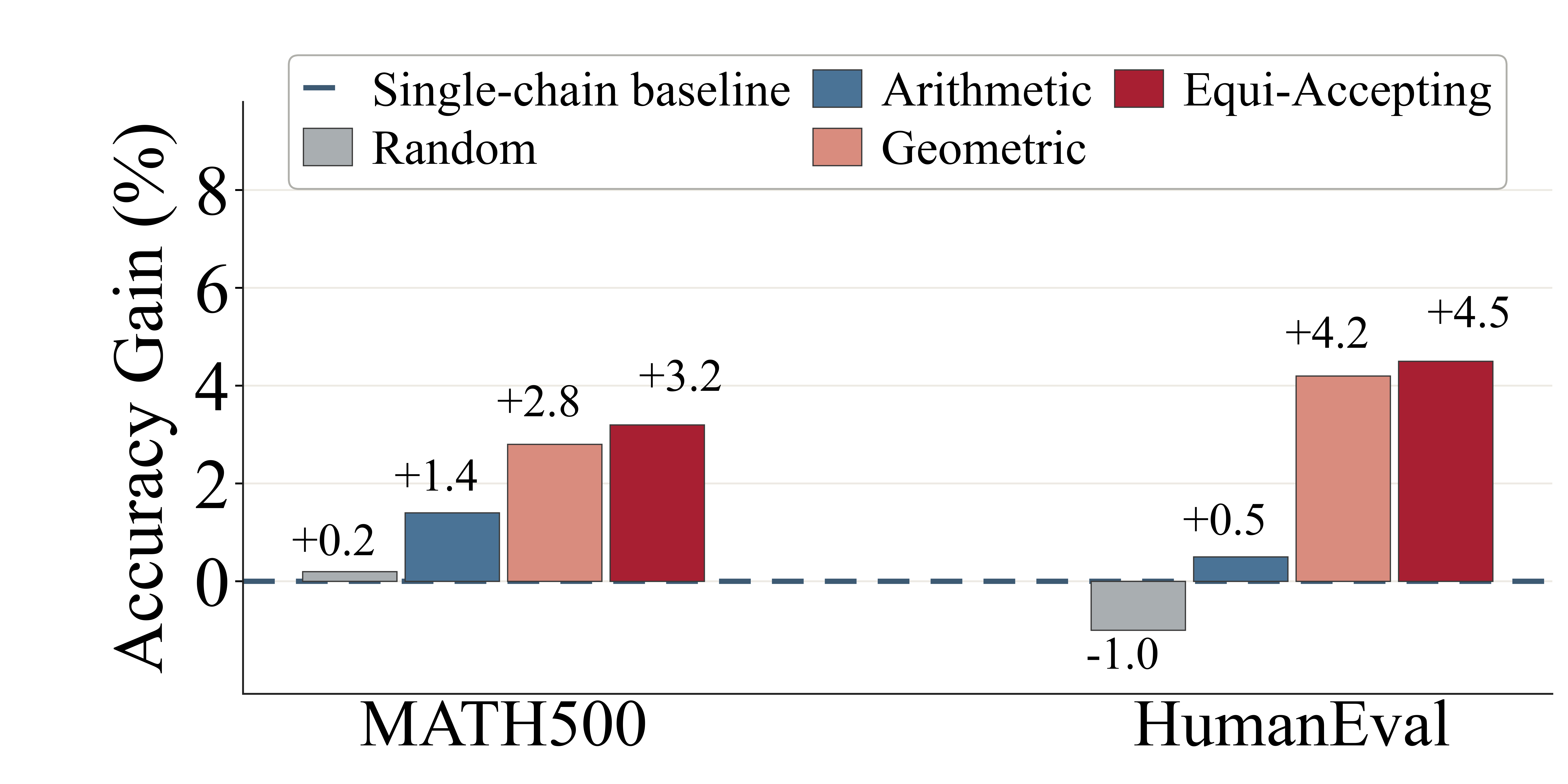}
    \end{minipage}
    \vspace{-.3cm}
    \caption{\textbf{Replica count and ladder design.}
    \emph{Left:} (Top): accuracy gain as the number of replicas increases, and (Bottom): the corresponding compute in wall-clock time per example; the favorable scaling reflects the batching of replicas in parallel.
    \emph{Right:} effect of ladder design on accuracy across two datasets, MATH500 and HumanEval, with $K=5$ replicas.}
    \label{fig:replica_tradeoff}
\end{figure}



\noindent\textbf{Ladder design.}
We examine the effect of appropriate ladder design in the efficacy of the model. We compare: i) randomly picked ladder temperatures, ii) arithmetic sequences , iii) geometric sequences, and iv) equi-accepting sharpening ladders with $K=5$ and fixed block size and decoding budget (Figure~\ref{fig:replica_tradeoff}, right).
Ladder spacing controls how readily states move between exploratory and strongly sharpened replicas.
Random and arithmetic spacing offer limited gains, with random spacing even harming HumanEval.
Geometric spacing, which uses equal power ratios (Remark~\ref{rem:ar_geom_seq}), improves both benchmarks, while equi-accepting performs best.
By approximately equalizing adjacent swap-acceptance rates, the latter aims to prevent any single pair from becoming a communication bottleneck. These results suggest that balancing communication across the ladder helps replicas share useful exploration more effectively.

\flushbottom

\section{Conclusion}
We introduced \method{}, a parallel tempering sampler for power-sharpened LLM distributions that addresses the exploration--exploitation trade-off of single-chain power sampling.
Low-power replicas explore diverse reasoning trajectories, high-power replicas refine them, and swaps couple the ladder power moving promising traces toward the sharpest rung at negligible cost.
In our analysis, we identified a structural truncation bias in early-stopped power samplers and removed it with a fixed-horizon construction that restores preservation of the true sharpened target, and provides all chains the same exploration bandwidth.
Empirically, \method{} is best or tied-best in all $15$ model--benchmark settings and outperforms GRPO without parameter updates or rewards, with gains that come from the swaps rather than extra compute or chains.
With Qwen3.5-9B, it performs comparably to frontier models, suggesting that small frozen models can reason without additional parameter updates.


\section*{Acknowledgment}
We thank the CoreWeave AI cloud platform for supporting part of the computational resources
used in this work. Nan Jiang acknowledges support from the Texas Advanced Computing Center (TACC) under award CCR25054.
E.A. Theodorou and P. Theodoropoulos were partly supported by the Defense Advanced Research Projects Agency (DARPA) through the Artificial Intelligence Quantified (AIQ) program, under Cooperative Agreement No. HR00112520010. The views, opinions, and/or findings expressed are those of the authors and should not be interpreted as representing the official views or policies of the Department of Defense or the U.S. Government.

\newpage
\bibliography{iclr2027_conference}

\newpage
\appendix
\begingroup
\setcounter{tocdepth}{2}
\setlength{\cftbeforetoctitleskip}{0pt}
\setlength{\cftaftertoctitleskip}{6pt}
\setlength{\cftbeforesecskip}{2pt}
\renewcommand{\cftsubsecfont}{\small}
\renewcommand{\cftsubsecpagefont}{\small}
\tableofcontents
\endgroup
\allowdisplaybreaks
\newpage


\appendix
\section{Notation and Preliminaries}
\label{sec:theory}

This section establishes the fixed-horizon representation used throughout the
paper. It introduces prompt-conditioned completion records, the sequence-level
power target, the tokenwise-powered proposal distribution, and the sequence of
stage horizons used by the progressive sampler. 

\subsection{Prompt-conditioned autoregressive model}
Fix a tokenized prompt $\mathbf{x}_0$, a finite language-model vocabulary
$\mathcal V_{\mathrm{LM}}$, a set of configured terminal tokens
$\mathcal E\subseteq\mathcal V_{\mathrm{LM}}$, and a maximum completion
horizon $T\geq1$.  Let $\bot\notin\mathcal V_{\mathrm{LM}}$ be an auxiliary post-terminal symbol used to express the fixed-horizon state space, and define the augmented vocabulary
\[
    \mathcal V
    :=
    \mathcal V_{\mathrm{LM}}\cup\{\bot\}.
\]
The symbol \(\bot\) serves exclusively as bookkeeping for the configured
post-terminal convention. A \emph{completion record} stores the generated
completion tokens, while the prompt \(\mathbf{x}_0\) enters every model distribution as
fixed conditioning context.
For every supported pre-terminal history $h$, the language model defines
normalized next-token probabilities
\(
    p_{\mathrm{LM}}(v\mid \mathbf{x}_0,h),
    \,
    v\in\mathcal V_{\mathrm{LM}}.
\)
We define a completion history \emph{open} through the position preceding its
first configured terminal token and \emph{terminated} from that token onward.
The configured post-terminal convention extends the model distribution to
\(\mathcal V\) through
\begin{equation}
p_0(v\mid \mathbf{x}_0,h)
:=
\begin{cases}
p_{\mathrm{LM}}(v\mid \mathbf{x}_0,h),
&
h\text{ is supported and open and }v\in\mathcal V_{\mathrm{LM}},
\\[1.5mm]
1,
&
h\text{ is supported and terminated and }v=\bot,
\\[1mm]
0,
&
\text{in all remaining cases.}
\end{cases}
\label{eq:theory-absorbing-law}
\end{equation}

In Eq.~\ref{eq:theory-absorbing-law}, invalid constructions, 
such as a $\bot$ before the first terminal token or a 
non-$\bot$ token after termination, are assigned 
zero probability. Generation ends at the first configured terminal token. The post-terminal
convention then fills the remaining horizon with $\bot$, each padding position
contributing a factor of one.

A fixed-horizon record $\mathbf{x} = x_{1:T} \in \mathcal{V}^T$ then has
prompt-conditioned probability
\begin{equation}
    p_0(\mathbf{x} \mid \mathbf{x}_0)
    := \prod_{t=1}^{T} p_0\bigl(x_t \mid \mathbf{x}_0, \mathbf{x}_{<t}\bigr),
    \label{eq:record-prob}
\end{equation}
where each factor is given by Eq.~\ref{eq:theory-absorbing-law}, so any record
containing an invalid construction has zero probability.
The positive-support state space at horizon \(T\) is denoted
\begin{equation}
    \mathcal S_T
    :=
    \left\{
        \mathbf{x}\in\mathcal V^T:
        p_0(\mathbf{x}\mid \mathbf{x}_0)>0
    \right\}.
    \label{eq:theory-positive-support}
\end{equation}
For a sequence \(\mathbf{x}\), we define its first terminal position by 
$$   
    \tau(\mathbf{x})
    :=
    \min\{t\in\{1,\ldots,T\}:x_t\in\mathcal E\},
$$
with its completion length given by
$   
    \operatorname{len}(\mathbf{x})
    :=
    \min\{\tau(\mathbf{x}),T\}.
$
For \(\mathbf{x}\in\mathcal S_T\), when \(\tau(\mathbf{x})\leq T\), the fixed-horizon state
satisfies
\(
    \mathbf{x}_{\tau(\mathbf{x})+1:T}
    =
    (\bot,\ldots,\bot),
\)
and its probability simplifies to
\begin{equation}
    p_0(\mathbf{x}_{1:T}\mid \mathbf{x}_0)
    =
    \prod_{t=1}^{\tau(\mathbf{x})}
    p_{\mathrm{LM}}(x_t\mid \mathbf{x}_0,\mathbf{x}_{<t}).
    \label{eq:theory-terminated-sequence-law}
\end{equation}
For \(\mathbf{x}\in\mathcal S_T\), when no terminal token is emitted, the language model generates
all \(T\) positions and the completion is right-censored at the declared
horizon.
The implementation supports both explicit deterministic tails and equivalent
implicit-tail representations. In the latter representation, the stored
completion record together with the configured post-terminal convention
uniquely specifies the corresponding element of \(\mathcal S_T\). During
progressive construction, an open completion record of length below \(T\)
serves as an intermediate record, and generation through the remaining
positions produces its fixed-horizon representation.

\subsection{Sequence-level power target}

For a sharpening power \(\alpha\geq 1\), we define the sequence-level power
target on \(\mathcal S_T\) by
\begin{equation}
    \pi_{\alpha,T}(\mathbf{x}\mid \mathbf{x}_0)
    :=
    \frac{
        p_0(\mathbf{x}\mid \mathbf{x}_0)^\alpha
    }{
        Z_{\alpha,T}(\mathbf{x}_0)
    },
    \qquad
    Z_{\alpha,T}(\mathbf{x}_0)
    :=
    \sum_{\mathbf{z}\in\mathcal S_T}
        p_0(\mathbf{z}\mid \mathbf{x}_0)^\alpha .
    \label{eq:theory-power-target}
\end{equation}
The finite state space and its positive support give
\(0<Z_{\alpha,T}(\mathbf{x}_0)<\infty\), so
Equation~\ref{eq:theory-power-target} defines a probability distribution on
the fixed-horizon completion space.
Following the above, we define the sequence energy as
$
    U_T(\mathbf{x})
    :=
    -\log p_0(\mathbf{x}\mid \mathbf{x}_0).
$
The target then takes the Gibbs form
\[
    \pi_{\alpha,T}(\mathbf{x}\mid \mathbf{x}_0)
    =
    Z_{\alpha,T}(\mathbf{x}_0)^{-1}
    \exp\{-\alpha U_T(\mathbf{x})\}.
\]

Now consider a power ladder with
\(
    1\leq\alpha_1<\alpha_2<\cdots<\alpha_K,
\)
associating a replica with each power level. The joint replica state is given by
\(
    \mathbf X
    =
    \left(
        \mathbf{x}^{(1)},\ldots,\mathbf{x}^{(K)}
    \right)
    \in\mathcal S_T^K,
\)
with its joint product target being written as
\begin{equation}
    \Pi_T(\mathbf X\mid \mathbf{x}_0)
    :=
    \prod_{k=1}^{K}
    \pi_{\alpha_k,T}
    \left(\mathbf{x}^{(k)}\mid \mathbf{x}_0\right).
    \label{eq:theory-product-target}
\end{equation}
Each power remains associated with its rung index. A predesignated output rung
\(k^\star\) therefore has target marginal $ \pi_{\alpha_{k^\star},T}$.
The choice of \(k^\star\) is fixed independently of the realized completion
records. The horizon \(T\) forms part of the target specification, and every target
statement in this appendix refers to this declared finite horizon. For
\(K=1\), the product target reduces to a single sequence-level target and the
adjacent-exchange operation becomes the identity.

\subsection{Sequence-level powering is not tokenwise temperature decoding}
\label{ps-vs-lowtemp}
Sequence-level powering generally differs from tokenwise temperature decoding
because its next-token conditional depends on possible continuations; see
\citet[Section~4.1]{tomihari2026power} for the analysis and derivation.
For a supported prefix \(h_t:=(\mathbf{x}_0,\mathbf{x}_{<t})\), define the tokenwise powered
normalizer and proposal used by our sampler:
\begin{align}
    z_{\alpha,t}(h_t)
    &:=
    \sum_{v\in\mathcal V}
        p_0(v\mid h_t)^\alpha,
    \label{eq:theory-token-normalizer}
    \\
    g_{\alpha,t}(v\mid h_t)
    &:=
    \frac{
        p_0(v\mid h_t)^\alpha
    }{
        z_{\alpha,t}(h_t)
    }.
    \label{eq:theory-tokenwise-proposal}
\end{align}
Conditioning on \(h_t\) abbreviates conditioning on \((\mathbf{x}_0,\mathbf{x}_{<t})\)
throughout. After a terminal token, \(g_{\alpha,t}\) places unit mass on
\(\bot\), following the configured post-terminal convention.
In our notation, the powered continuation mass is
\begin{equation}
    H_{\alpha,t}(v\mid h_t)
    :=
    \begin{cases}
    \displaystyle
    \sum_{\mathbf{u}_{t+1:T}\in\mathcal V^{T-t}}
        p_0(\mathbf{u}_{t+1:T}\mid h_t,v)^\alpha,
    &
    p_0(v\mid h_t)>0,
    \\[3mm]
    0,
    &
    p_0(v\mid h_t)=0.
    \end{cases}
    \label{eq:theory-continuation-mass}
\end{equation}
At $t = T$, the continuation is empty and carries mass one, so
$H_{\alpha,T}(v \mid h_T) = 1$ whenever $p_0(v \mid h_T) > 0$.
Therefore, the sequence-level target induces the next-token conditional
\begin{equation}
    \pi_{\alpha,T}(v\mid h_t)
    =
    \frac{
        p_0(v\mid h_t)^\alpha
        H_{\alpha,t}(v\mid h_t)
    }{
        \displaystyle
        \sum_{u\in\mathcal V}
        p_0(u\mid h_t)^\alpha
        H_{\alpha,t}(u\mid h_t)
    }.
    \label{eq:theory-power-conditional}
\end{equation}
The tractable proposal \(g_{\alpha,t}\) omits the continuation factor
\(H_{\alpha,t}\) and equals this conditional only when that factor is constant
over supported next tokens.

\section{Implementation Details}
\label{sec:implementation-details}
\label{apx:implement}

A naive implementation runs the $K$ replicas sequentially, giving $K$-fold wall-clock time
increase for $K$ replicas.
Instead, we execute block extensions and suffix-resampling
proposals for all replicas jointly on a paged-attention engine (vLLM), while
preserving each replica's proposal temperature $\tau_k=1/\alpha_k$. Continuous
batching handles the variable-length suffix proposals with little padding, and
KV-cache reuse allows proposed suffixes to be generated and scored using cached prefix states. In other words, each replica stores its current sequence, cached base-model log-probability, and KV-cache
handle, so an accepted swap merely exchanges record and cache without
additional model evaluation. Total generated tokens still grow with $K$, but
per-refinement-round wall-clock scales more favorably, until GPU saturation, after which
scaling approaches the $K$-fold limit. This optimization comes solely from parallel execution
and leaves the local MH and swap kernels unchanged.

The implementation combines three operations at each progressive stage:
block extension, local refinement, and parallel tempering swaps. The first
operation carries every completion record to the current stage horizon. The
second operation refreshes a suffix of each record through a
Metropolis--Hastings transition. Lastly, the third operation swaps whole records
between adjacent powers.

\subsection{Step 1: Block extension}
\label{subsec:implementation-extension}

Our implementation follows the blockwise construction of
\citep{karan2025reasoning}. We define $B\geq1$ to be the block width and further write
\(
    M:=\left\lceil\frac{T}{B}\right\rceil,
    \quad
    T_m:=\min\{mB,T\},
    \quad
    m=0,\ldots,M,
\)
with $T_0=0$ and $T_M=T$. All fixed-horizon definitions apply at
\(T_m\) by replacing \(T\) with \(T_m\). In particular, we define the block-wise state space \(\mathcal S_{T_m}\), with the 
terminal position and length on that space.
For stage $m$ and rung $k$, abbreviate
\begin{equation}
    \pi_{m,k}(\cdot):=\pi_{\alpha_k,T_m}(\cdot\mid \mathbf{x}_0),
    \qquad
    \Pi_m(\mathbf X):=\Pi_{T_m}(\mathbf X\mid \mathbf{x}_0)
    =\prod_{k=1}^{K}\pi_{m,k}\bigl(\mathbf{x}^{(k)}\bigr).
    \label{eq:theory-stage-targets}
\end{equation}
Thus, we can infer that the terminal-stage target is $\Pi_M=\Pi_T$.
For the rest of this section, the fixed conditioning on \(\mathbf{x}_0\) is suppressed in stage
notation. 

Block extension from \(T_{m-1}\) to \(T_m\) initializes the state on the
enlarged space. Sampling new tokens from
\(g_{\alpha_k,t}\) are generally not distributed from the target
law of \(\pi_{m,k}\).
Therefore, we can interpret block extension as a warm start for the subsequent
Metropolis--Hastings refinement. A finite number of refinement steps produces an
approximation to \(\Pi_m\). The exact invariance and convergence results
below concern the local and parallel tempering kernels after the stage horizon
has been fixed.

A record is active at the start of block extension exactly when its current
state is open. Activity is recomputed after all preceding local refinements
and parallel tempering swaps. Thus, if an accepted refinement removes a terminal
token, the record becomes active and is extended at the next stage.
A currently terminated record is extended deterministically through its
post-terminal tail.
At rung \(k\), every model-generated token is drawn according to
\begin{equation}
    x_t^{(k)}
    \sim
    g_{\alpha_k,t}
    \left(
        \,\cdot\mid \mathbf{x}_0,\mathbf{x}_{<t}^{(k)}
    \right),
    \qquad
    t=T_{m-1}+1,\ldots,T_m.
    \label{eq:implementation-extension}
\end{equation}
The generation for each rung stops after the first terminal token, and following the post-terminal convention the remaining positions are filled up to \(T_m\) deterministically. Requests may be batched across prompts and rungs without altering the record-specific law in Equation~\ref{eq:implementation-extension}.

For each model-generated token, the implementation stores its base-model log probability
\begin{equation}
    \ell_t(\mathbf{x})
    :=
    \log p_0(x_t\mid \mathbf{x}_0,\mathbf{x}_{<t}).
    \label{eq:implementation-base-log-probability}
\end{equation}
For every rung \(j\), the corresponding tokenwise normalizer is
\begin{equation}
    \zeta_{t,j}(\mathbf{x})
    :=
    \log z_{\alpha_j,t}(\mathbf{x}_0,\mathbf{x}_{<t}),
    \qquad
    j=1,\ldots,K.
    \label{eq:implementation-token-normalizers}
\end{equation}
Finally, deterministic post-terminal positions satisfy
\(
    \ell_t(\mathbf{x})=\zeta_{t,j}(\mathbf{x})=0,
\)
so their tails may be stored explicitly or reconstructed when evaluating fixed-horizon sums.

\subsection{Step 2: Local refinement}
\label{subsec:implementation-local}

After extension, each rung receives its configured number of local updates.
Fix a stage $m$ and rung $k$. The implementation uses the state-independent
uniform restart law
\[
    r\sim\omega_{m,k},
    \qquad
    \omega_{m,k}(r)=\frac{1}{T_m},
    \qquad
    r\in\{1,\ldots,T_m\}.
\]
The theory below allows any state-independent restart law and requires
\(\omega_{m,k}(1)>0\) only for convergence. Conditioned on \(r\), the
proposal retains the prefix before \(r\) and regenerates the suffix through
the common endpoint \(T_m\):
\begin{equation}
    q_{m,k,r}(\mathbf{y}\mid \mathbf{x})
    :=
    \mathbf 1\{\mathbf{y}_{<r}=\mathbf{x}_{<r}\}
    \prod_{t=r}^{T_m}
    g_{\alpha_k,t}
    \left(
        y_t\mid \mathbf{x}_0,\mathbf{y}_{<t}
    \right).
    \label{eq:implementation-fixed-suffix-proposal}
\end{equation}
If \(r>\tau(\mathbf{x})\), the retained prefix is already terminated and the proposal
is a self-transition. The implementation detects this from the stored
terminal position and issues no model call. Nevertheless, the
restart law is applied over all of \(\{1,\ldots,T_m\}\)
rather than over \(\{1,\ldots,\tau(\mathbf{x})\}\). Restricting it to the open
positions would make \(\omega\) depend on the current state and would
require a \(\tau(\mathbf{x})/\tau(\mathbf{y})\) factor in the acceptance ratio. If for the new position
\(r\leq\tau(\mathbf{x})\), the proposal move could remove, replace, or relocate
the terminal token. Because all records lie in \(\mathcal S_{T_m}\), both
directions use the same restart set and are evaluated through the same horizon;
post-terminal factors equal one.

The proposal is accepted with probability
\begin{equation}
    A_{m,k,r}(\mathbf{x},\mathbf{y})
    :=
    1\wedge
    \frac{
        \pi_{m,k}(\mathbf{y})q_{m,k,r}(\mathbf{x}\mid \mathbf{y})
    }{
        \pi_{m,k}(\mathbf{x})q_{m,k,r}(\mathbf{y}\mid \mathbf{x})
    }.
    \label{eq:implementation-local-mh-acceptance}
\end{equation}
The usual MH convention is used when the reverse proposal probability
vanishes.
The factor \(\omega_{m,k}(r)\) cancels because the restart law is independent
of the current state. For supported \(\mathbf{x},\mathbf{y}\) with \(\mathbf{x}_{<r}=\mathbf{y}_{<r}\), the
powered base-model terms also cancel, as shown next.
The resulting cached acceptance ratio is
\begin{equation}
    \log R_{m,k,r}^{\mathrm{loc}}(\mathbf{x},\mathbf{y})
    :=
    \log
    \frac{
        \pi_{m,k}(\mathbf{y})\,
        q_{m,k,r}(\mathbf{x}\mid \mathbf{y})
    }{
        \pi_{m,k}(\mathbf{x})\,
        q_{m,k,r}(\mathbf{y}\mid \mathbf{x})
    }
    =
    \sum_{t=r}^{T_m}
    \left[
        \zeta_{t,k}(\mathbf{y})-\zeta_{t,k}(\mathbf{x})
    \right].
\label{eq:implementation-local-cached-ratio}
\end{equation}
An accepted proposal replaces the suffix and all aligned
\(\ell\)- and \(\zeta\)-entries. A rejected proposal retains the current
completion record and its cache.
Updates across rungs use independent auxiliary randomness but may be batched
into a shared model evaluation.

\subsection{Step 3: Swap exchange}
\label{subsec:implementation-swap}

Local refinement updates the completion records within their current rungs,
while parallel tempering swaps transport complete records across the power
ladder. For \(k\in\{1,\ldots,K-1\}\), let \(\sigma_k\) denote the
transposition of coordinates \(k\) and \(k+1\). The proposed exchange is
accepted with probability
\begin{equation}
    A_{m,k}^{\mathrm{swap}}(\mathbf X)
    :=
    1\wedge
    \frac{
        \Pi_m(\sigma_k\mathbf X)
    }{
        \Pi_m(\mathbf X)
    }.
    \label{eq:implementation-swap-general}
\end{equation}
Writing \(\mathbf{x}:=\mathbf{x}^{(k)}\) and \(\mathbf{y}:=\mathbf{x}^{(k+1)}\), the log acceptance ratio is
\begin{align}
    \log
    \frac{
        \Pi_m(\sigma_k\mathbf X)
    }{
        \Pi_m(\mathbf X)
    }
    &=
    (\alpha_{k+1}-\alpha_k)
    \left[
        \log p_0(\mathbf{x}\mid \mathbf{x}_0)
        -
        \log p_0(\mathbf{y}\mid \mathbf{x}_0)
    \right].
    \label{eq:swap-acceptance-theory}
\end{align}
The cached base-model scores satisfy
\[
    \log p_0(\mathbf{x}\mid \mathbf{x}_0)
    =
    \sum_{t=1}^{T_m}\ell_t(\mathbf{x}),
\]
so the cached completion records supply the full exchange ratio.

An ordered adjacent sweep attempts exchanges for
\(k=1,\ldots,K-1\) in increasing order, so each attempt acts on the result of
the preceding ones. The powers remain attached to their rung indices, while an
accepted exchange moves the complete record and every aligned cache. For
\(K=1\), no exchange is attempted.

\subsection{Putting the implementation together}
\label{subsec:implementation-complete}

Let \(n_{\mathrm{local},m}\) denote the number of synchronized local rounds in
one schedule period, and let \(n_{\mathrm{sweep},m}\) denote the number of
ordered adjacent sweeps that follow. At stage \(m\), block extension first
advances every record from \(T_{m-1}\) to \(T_m\). The sampler then repeats
\(N_m\) schedule periods, each consisting of
\(n_{\mathrm{local},m}\) local rounds followed by
\(n_{\mathrm{sweep},m}\) ordered sweeps. A zero count skips the corresponding
operation. After stage \(M\), the predesignated rung \(k^\star\) supplies the
returned record, which is presented through its first configured terminal
token or through \(T\) if it is right-censored.

Algorithm~\ref{alg:PPT} summarizes this execution order. The formal Markov
kernels induced by these operations are introduced once, in
Section~\ref{sec:theoretical-properties}.
\newcommand{\LeftComment}[1]{%
    \(\triangleright\)~#1%
}

\begin{algorithm}[t]
\caption{Progressive fixed-horizon parallel tempering}
\label{alg:PPT}
\begin{algorithmic}[1]
\Require Fixed prompt \(\mathbf{x}_0\), horizon \(T\), block width \(B\), ladder powers
    \(\alpha_{1:K}\), stage counts \(N_{1:M}\), update schedules
    \(n_{\mathrm{local},1:M}\) and \(n_{\mathrm{sweep},1:M}\), output rung
    \(k^\star\)
\Ensure A completion from rung \(k^\star\)

\State Initialize records \(\mathbf{x}^{(1:K)}\) and their aligned caches as empty.

\For{stage \(m=1,\ldots,M\)}
    \State Set the current horizon \(T_m\gets\min\{mB,T\}\).

    \Statex
    \Statex  \LeftComment{\textbf{\textcolor{blue}{Step 1: Block extension.}} Expose the next block of tokens and extend each replica to \(T_m\).}
    \For{\(k=1,\ldots,K\)}
        \While{\(\mathbf{x}^{(k)}\) is open and \(|\mathbf{x}^{(k)}|<T_m\)}
            \State Sample the next token from
            \(g_{\alpha_k,t}(\cdot\mid \mathbf{x}_0,\mathbf{x}_{<t}^{(k)})\).
            \State Append the token and its cached \(\ell\)- and
            \(\zeta\)-values to \(\mathbf{x}^{(k)}\).
        \EndWhile
        \State Apply the post-terminal convention through \(T_m\).
    \EndFor

    \For{\(n=1,\ldots,N_m\)}
        \Statex
        \Statex  \LeftComment{\textbf{\textcolor{blue}{Step 2: Local refinement.}} Resample suffixes within each rung, preserving its target distribution.}
        \For{\(a=1,\ldots,n_{\mathrm{local},m}\)}
            \For{\(k=1,\ldots,K\)}
                \State Sample \(r\sim\omega_{m,k}\).
                \State Propose \(\mathbf{y}\) by resampling the corresponding suffix
                through \(T_m\) from
                \(q_{m,k,r}(\cdot\mid \mathbf{x}^{(k)})\).
                \State With probability \(A_{m,k,r}(\mathbf{x}^{(k)},\mathbf{y})\), replace
                \(\mathbf{x}^{(k)}\) and its cache by \(\mathbf{y}\) and its cache.
            \EndFor
        \EndFor

        \Statex
        \Statex \LeftComment{\textbf{\textcolor{blue}{Step 3: Replica swaps.}} Exchange adjacent rungs so records can move between power levels.}
        \For{\(b=1,\ldots,n_{\mathrm{sweep},m}\)}
            \For{\(k=1,\ldots,K-1\)}
                \State With probability
                \(A_{m,k}^{\mathrm{swap}}(\mathbf X)\), swap replicas
                \(k\) and \(k+1\), including their caches.
            \EndFor
        \EndFor
    \EndFor
\EndFor

\State Let \(\mathbf{x}\gets \mathbf{x}^{(k^\star)}\).
\State \Return \(\mathbf{x}\) through its first configured terminal token, or all of
    \(\mathbf{x}\) through \(T\) if no such token occurs.
\end{algorithmic}
\end{algorithm}

\section{Theoretical Properties: Bias and Exactness of PPT}
\label{sec:theoretical-properties}

In this section, we study the exactness of \method. 
First, we define 
the complete transition kernel and show how one-way truncating local updates
lead to structural bias in both the ensemble and the 
returned rung, despite the chain swaps. 
Then, we establish that \method can preserve the true target for fixed-horizon MH updates
and prove asymptotic convergence to the intended target under supported full restarts.

\subsection{Fixed-horizon targets and PPT kernels}
\label{subsec:theory-ppt-setting}

We start by fixing a stage $m$, after block extension and hold its horizon \(T_m\) fixed.
We will suppress the stage index throughout for notation simplicity.
Therefore, denote by \(\mathcal S_T\) the
space of records with horizon \(T\), including
records that terminate earlier and are padded with \(\bot\). Subsequently, we denote by
\(\operatorname{len}(\mathbf{x})\) the number of generated tokens before padding, including
the terminal token when present. An unterminated record in
\(\mathcal S_T\) has length \(T\).

Then, at rung \(k\), we write the marginal target as \(\pi_k=\pi_{\alpha_k,T}\)
defined in Eq.~\ref{eq:theory-stage-targets}.
The ensemble
state and target are
\[
\mathbf X=(\mathbf{x}^{(1)},\ldots,\mathbf{x}^{(K)}),
\qquad
\Pi(\mathbf X)=\prod_{k=1}^K\pi_k(\mathbf{x}^{(k)}).
\]
Lastly, for any joint law \(\mu\), we express its marginal at
rung \(k\) as $\mu^{(k)}$. Distributional deviations are measured by the total variation,
\[
\operatorname{TV}(\mu,\nu)
=\frac12\sum_{\mathbf{x}}|\mu(\mathbf{x})-\nu(\mathbf{x})|
=\sup_A|\mu(A)-\nu(A)|.
\]
The sum is over the common finite space of the two laws. Results for
\(\mu^{(k)}\) apply directly when rung \(k\) is selected as the output.

\paragraph{Local refinement.}

For a fixed prompt \(\mathbf{x}_0\), rung \(k\), and restart position
\(r\in\{1,\ldots,T\}\), the proposal retains positions before \(r\)
and regenerates the remaining record:
\[
q_{k,r}(\mathbf{y}\mid \mathbf{x})
=\mathbf1\{\mathbf{y}_{<r}=\mathbf{x}_{<r}\}
\prod_{t=r}^T g_{\alpha_k,t}(y_t\mid \mathbf{x}_0,\mathbf{y}_{<t}).
\]
Here \(\mathbf{y}_{<r}=(y_1,\ldots,y_{r-1})\), and
\(g_{\alpha_k,t}\) is the normalized token proposal defined in
Subsection~\ref{subsec:implementation-local}. Once a terminal token
appears, each remaining padding token has proposal probability one.
Thus forward and reverse proposals share the endpoint \(T\), even
when the two records terminate at prior positions.
For a pair with \(q_{k,r}(\mathbf{y}\mid \mathbf{x})>0\), the MH acceptance probability is
\[
A_{k,r}(\mathbf{x},\mathbf{y})
=1\wedge
\frac{\pi_k(\mathbf{y})q_{k,r}(\mathbf{x}\mid \mathbf{y})}
{\pi_k(\mathbf{x})q_{k,r}(\mathbf{y}\mid \mathbf{x})}.
\]
A proposed move with zero reverse probability is rejected. Set
\(A_{k,r}(\mathbf{x},\mathbf{y})=0\) when \(q_{k,r}(\mathbf{y}\mid \mathbf{x})=0\), since such a pair
is never proposed. The resulting transition kernel is
\[
L_{k,r}(\mathbf{x},\mathbf{y})=
\begin{cases}
q_{k,r}(\mathbf{y}\mid \mathbf{x})A_{k,r}(\mathbf{x},\mathbf{y}),&\mathbf{y}\ne \mathbf{x},\\[1mm]
1-\displaystyle\sum_{\mathbf{z}\ne \mathbf{x}}q_{k,r}(\mathbf{z}\mid \mathbf{x})A_{k,r}(\mathbf{x},\mathbf{z}),&\mathbf{y}=\mathbf{x}.
\end{cases}
\]
Its diagonal includes both self-proposals and rejected proposals.

Restart positions have state-independent probabilities
\(\omega_k(r)\ge0\), with \(\sum_{r=1}^T\omega_k(r)=1\). The
rung-level kernel and one synchronized local sweep are 
\[
L_k(\mathbf{x},\mathbf{y})=\sum_{r=1}^T\omega_k(r)L_{k,r}(\mathbf{x},\mathbf{y}),
\qquad
\mathcal R(\mathbf X,\mathbf Y)
=\prod_{k=1}^K L_k(\mathbf{x}^{(k)},\mathbf{y}^{(k)}).
\]
The product expresses conditional independence of the complete local
updates across rungs.

\paragraph{Replica exchange.}
Let \(\sigma_k\mathbf X\) exchange coordinates \(k\) and \(k+1\).
The adjacent rung MH acceptance probability is given by
\[
\begin{aligned}
A_k^{\mathrm{swap}}(\mathbf X)
&=1\wedge\frac{\Pi(\sigma_k\mathbf X)}{\Pi(\mathbf X)}
=1\wedge
\frac{\pi_k(\mathbf{x}^{(k+1)})\pi_{k+1}(\mathbf{x}^{(k)})}
{\pi_k(\mathbf{x}^{(k)})\pi_{k+1}(\mathbf{x}^{(k+1)})},
\end{aligned}
\]
defining the swap kernel as follows:
\[
W_k(\mathbf X,\mathbf Y)
=A_k^{\mathrm{swap}}(\mathbf X)\mathbf1\{\mathbf Y=\sigma_k\mathbf X\}
+[1-A_k^{\mathrm{swap}}(\mathbf X)]\mathbf1\{\mathbf Y=\mathbf X\}.
\]
Therefore, the ordered adjacent sweep is defined as the composition of 
all $W_k$, namely:
\(\mathcal C=W_1\cdots W_{K-1}\). Kernels act on laws from the right,
so \(W_1\) is applied first and each exchange acts on the records
produced by the preceding exchanges.

\paragraph{Complete PPT transition.}
Fix integers \(n_{\mathrm{local}}\ge1\) and
\(n_{\mathrm{sweep}}\ge0\). One refinement round has kernel
\[
P=\mathcal R^{\,n_{\mathrm{local}}}
\mathcal C^{\,n_{\mathrm{sweep}}}.
\]
Thus \(n_{\mathrm{local}}\) counts synchronized local sweeps per round,
and \(n_{\mathrm{sweep}}\) counts adjacent exchange sweeps. Starting
from a law \(\mu_0\), the law after \(N\) rounds is
\(\mu_N=\mu_0P^N\). The horizon, targets, restart probabilities, and
schedule remain fixed during these rounds. Block extension changes
the state space and supplies the initializer; it is not part of \(P\).

\subsection{Structural bias from one-way truncation in PPT}
\label{sec:truncation-lower-bound}

The structural bias arises when a record's current length becomes the
generation budget for subsequent local proposals. For instance, if 
a proposal replaces a $10$-token sequence with one that terminates after
$4$ tokens, then next proposal is capped at $4$ tokens. It would be unable to
reconstruct the original $10$-token record, implying that the reverse proposal has
zero probability. Computing the acceptance ratio only using token scores
through the shorter endpoint can miss this loss of reverse support and
accept the shortening move. The exact MH rule would reject it. Accepting
shortening moves while rejecting length-increasing moves
results in one-way probability leakage from longer records
to shorter ones with no return flow.

In PPT, a swap can give a rung a longer record from another rung, but
it only moves a record that already exists. Hence, accepting shortening
moves at the target violates target invariance, with parallel 
tempering unable to undo this violation.
Furthermore, under a uniform accepted-entry condition, we show that repeated shortening also
yields an asymptotic bias floor for the ensemble and returned rung.

\paragraph{Truncating kernel and comparison space.}
We denote the local transition kernel implemented
by the truncating sampler at rung \(k\) by $\widetilde L_k$, including acceptance and
rejection. Since truncation may produce unterminated records shorter
than \(T\), we compare the sampler and target on a common finite space
\(\widetilde{\mathcal S}_T\supseteq\mathcal S_T\) containing these
records. Additionally, each marginal target \(\pi_k\) is extended by zero outside \(\mathcal S_T\), and
\(\operatorname{len}(\mathbf{x})\) counts generated tokens before padding.

The synchronized kernel \(\widetilde{\mathcal R}\) applies the local
updates independently across rungs. A complete round \(\widetilde P\)
consists of \(n_{\mathrm{local}}\) such sweeps followed by
\(n_{\mathrm{sweep}}\) exchange sweeps \(\widetilde{\mathcal C}\):
\[
\widetilde{\mathcal R}(\mathbf X,\mathbf Y)
=\prod_{k=1}^K\widetilde L_k(\mathbf{x}^{(k)},\mathbf{y}^{(k)}),
\qquad
\widetilde P
=\widetilde{\mathcal R}^{\,n_{\mathrm{local}}}
\widetilde{\mathcal C}^{\,n_{\mathrm{sweep}}}.
\]
Here \(\widetilde{\mathcal C}\) only exchanges or retains records;
it need not preserve \(\Pi\). Starting from an initial law
\(\widetilde\mu_0\), the ensemble law after \(N\) complete rounds is
\(\widetilde\mu_N=\widetilde\mu_0\widetilde P^N\).

\paragraph{Short records and excluded target mass.}
Fix a length threshold \(1\le h<T\). We denote the set of records
of length at most \(h\) by \(\mathcal B\), and the target probability
of a longer record at rung \(k\) by \(b_k\):
\[
\mathcal B=\{\mathbf{x}\in\widetilde{\mathcal S}_T:\operatorname{len}(\mathbf{x})\le h\},
\qquad b_k=\pi_k(\mathcal B^c).
\]
An ensemble belongs to \(\mathcal B^K\) precisely when every record
is short. Under the product target \(\Pi\), this event has probability
\(\prod_k(1-b_k)\), so the target probability of at least one long
record is
\[
\Pi((\mathcal B^K)^c)=1-\prod_{k=1}^K(1-b_k).
\]
This joint target mass and the individual masses \(b_k\) will determine
the bias bounds for the ensemble and returned rung, respectively.

\paragraph{One-way drift.}
Assume local updates under $\widetilde P$ never increase record length. A short record then
remains in \(\mathcal B\) after every local update. To track shortening
across the ensemble, define \(V(\mathbf X)\) as the number of long records:
\[
V(\mathbf X)=\sum_{j=1}^K\mathbf1\{\mathbf{x}^{(j)}\notin\mathcal B\}
\]
Note, that local updates under $\widetilde P$ cannot increase this count, and that swaps with $\widetilde C$ preserve it.
For a record initially drawn from a law \(\nu\) on
\(\widetilde{\mathcal S}_T\), let \(D_k(\nu)\) be the probability
that an accepted local update at rung \(k\) crosses from
\(\mathcal B^c\) into \(\mathcal B\).
\[
D_k(\nu)=\sum_{\mathbf{x}\notin\mathcal B}
\nu(\mathbf{x})\widetilde L_k(\mathbf{x},\mathcal B).
\]

Then, for any initial law \(\nu\), one local update decreases the probability
of a long record by exactly the accepted crossing probability \(D_k(\nu)\):
\[
(\nu\widetilde L_k)(\mathcal B^c)
=\nu(\mathcal B^c)-D_k(\nu).
\]
Since no
record in \(\mathcal B\) can leave it under a local update, the
loss of mass from \(\mathcal B^c\) equals \(D_k(\nu)\).

Additionally, consider an ensemble initialized from \(\Pi\). 
We define $\varphi_k:=D_k(\nu)$, when the record is
drawn from the marginal target \(\pi_k\).
Summing over target marginals gives a decrease of
\(\sum_k\varphi_k\) in the expected count after the first local sweep.
This loss gives the following bounds on the expected count and the TV error after one complete round:
\[
\mathbb E_\Pi[V]-\mathbb E_{\Pi\widetilde P}[V]
\ge\sum_k\varphi_k,
\qquad
\operatorname{TV}(\Pi\widetilde P,\Pi)
\ge\frac1K\sum_k\varphi_k.
\]
Thus \(\sum_k\varphi_k>0\) implies \(\Pi\widetilde P\ne\Pi\).
Positive shortening flow at the target is sufficient to violate
invariance; no uniform shortening probability is needed. Crossings observed
under a transient law \(\nu\) estimate \(D_k(\nu)\), so they alone do
not establish the target flow \(\varphi_k\).

\paragraph{Uniform accepted entry.}
To obtain an asymptotic bias floor from any initializer, we additionally
assume that every long record becomes short with probability at least
\(\delta\in(0,1]\) in one local update, uniformly over records and
rungs. Together with the closure of \(\mathcal B\), this gives
\begin{equation}
\widetilde L_k(\mathbf{x},\mathcal B)=1\quad(\mathbf{x}\in \mathcal B),
\qquad
\widetilde L_k(\mathbf{x},\mathcal B)\ge\delta\quad(\mathbf{x}\notin \mathcal B).
\label{eq:ppt-local-confinement}
\end{equation}
The lower bound concerns accepted transitions: positive probability of
proposing a short record alone does not guarantee entry into \(\mathcal B\).

\begin{theorem}[Structural bias of truncating PPT]
\label{thm:ppt-structural-bias}
Under Eq.~\ref{eq:ppt-local-confinement}, after \(N\ge1\) rounds from
any initializer \(\widetilde\mu_0\), the ensemble and each rung \(k\)
satisfy the following TV lower bounds:
\begin{align}
\operatorname{TV}(\widetilde\mu_N,\Pi)
&\ge\left[1-\prod_{j=1}^K(1-b_j)
-K(1-\delta)^{N}\right]_+,
\label{eq:ppt-joint-bias-lower}\\
\operatorname{TV}(\widetilde\mu_N^{(k)},\pi_k)
&\ge\left[b_k-K(1-\delta)^{N}\right]_+,
\label{eq:ppt-output-bias-lower}
\end{align}
Here \([u]_+=\max\{u,0\}\), and the correction term
\(K(1-\delta)^{N}\) bounds the probability that any
long record remains. As this term vanishes, the excluded target masses
give the asymptotic bias floors:
\begin{align}
\liminf_{N\to\infty}\operatorname{TV}(\widetilde\mu_N,\Pi)
&\ge1-\prod_{j=1}^K(1-b_j),\\
\liminf_{N\to\infty}\operatorname{TV}(\widetilde\mu_N^{(k)},\pi_k)
&\ge b_k.
\end{align}
\end{theorem}

\begin{proof}
Let \(\mathbf X'\) be the ensemble after the MH-update per rung
from \(\mathbf X\). Each long record remains long with probability at
most \(1-\delta\), while short records remain short. Thus the expected
number of long records contracts as follows:
\[
\begin{aligned}
\mathbb E[V(\mathbf X')\mid\mathbf X]
&=\sum_{j=1}^K\widetilde L_j(\mathbf{x}^{(j)},\mathcal B^c)\\
&\le(1-\delta)\sum_{j=1}^K\mathbf1\{\mathbf{x}^{(j)}\notin \mathcal B\}
=(1-\delta)V(\mathbf X).
\end{aligned}
\]
Since swaps preserve \(V\), only local refinement can affect this count.
Therefore, repeating this contraction for $N $ iterations yields
\[
\mathbb E_{\widetilde\mu_N}V
\le K(1-\delta)^{N}.
\]
Now, we consider at least one long record at any rung \(k\) which means \(V\ge1\). 
Since $\mathbf{1}_{V\ge1}\le V$, the probability of this event is bounded by the expected count of long records
\[
\widetilde\mu_N^{(k)}(\mathcal B^c)
\le\widetilde\mu_N((\mathcal B^K)^c)
\le\mathbb E_{\widetilde\mu_N}V
\le K(1-\delta)^{N}.
\]
From joint target \(\Pi\), the probabilities of these event is \(b_k\) and
\(1-\prod_j(1-b_j)\), respectively. 

Computing the TV distance between the two events above yields
\begin{align}
    \operatorname{TV}(\widetilde\mu_N^{(k)},\pi_k)
&\ge
\left[b_k-\mathbb E_{\widetilde\mu_N}V\right]_+,\\
\operatorname{TV}(\widetilde\mu_N,\Pi)
&\ge
\left[1-\prod_{j=1}^K(1-b_j)
-\mathbb E_{\widetilde\mu_N}V\right]_+.
\end{align}
Substituting the preceding bound on
$\E_{\widetilde\mu_N}V$ proves the finite-run inequalities.
Then as $N\rightarrow\infty$, this expected count tends to zero, yielding the
stated asymptotic bias floors.
\end{proof}

The quantities \(b_k\) determine the structural bias floors, while
\(\delta\) controls how fast, they are approached.
The inequalities above quantify structural bias caused by violation of the target preservation
for the one-way truncating MH-update law.

\subsection{Target preservation and convergence of fixed-horizon PPT}
\label{sec:local-preservation}
\label{sec:fixed-horizon-convergence}

In our implementation, we eliminate the structural bias and one-way truncation
shown above, by keeping a common horizon \(T\) fixed
throughout refinement. Every record and both directions of every local
proposal use this endpoint, with deterministic padding after a terminal
token. Consequently, early termination does not reduce the budget for later
proposals, since a restart at or before the terminal position can regenerate
beyond it, up to \(T\), when proposal support permits. MH acceptance
computed on this common space preserves the intended targets, and
supported full restarts ensure convergence to them. This removes the
asymptotic truncation bias.

We now analyze \(P\) on \(\mathcal S_T^K\), where the intended target
is strictly positive and both proposal directions use the same horizon.

\begin{proposition}[Target preservation]
\label{prop:fixed-stage-invariance}
Each fixed-position kernel \(L_{k,r}\) is reversible with respect to
\(\pi_k\). The local sweep \(\mathcal R\), each swap \(W_k\), the
ordered swap sweep \(\mathcal C\), and the complete PPT kernel \(P\)
preserve \(\Pi\).
\end{proposition}

\begin{proof}
For distinct \(\mathbf{x},\mathbf{y}\), the accepted local flow is
\[
\pi_k(\mathbf{x})L_{k,r}(\mathbf{x},\mathbf{y})
=\min\{\pi_k(\mathbf{x})q_{k,r}(\mathbf{y}\mid \mathbf{x}),
\pi_k(\mathbf{y})q_{k,r}(\mathbf{x}\mid \mathbf{y})\},
\]
which is symmetric in \(\mathbf{x},\mathbf{y}\). Thus each fixed-position kernel
preserves \(\pi_k\), as does its state-independent mixture \(L_k\).
Independence across rungs then gives \(\Pi\mathcal R=\Pi\).
Similarly, the accepted swap flow is
\[
\Pi(\mathbf X)A_k^{\mathrm{swap}}(\mathbf X)
=\min\{\Pi(\mathbf X),\Pi(\sigma_k\mathbf X)\},
\]
which is symmetric under the exchange. Hence \(\Pi W_k=\Pi\).
Composing these invariant kernels proves \(\Pi\mathcal C=\Pi\) and
\(\Pi P=\Pi\). The ordered compositions need not be reversible.
\end{proof}

In contrast, the error of the one-way truncation violates target invariance. If
\(\sum_k b_k>0\), then starting from \(\Pi\) gives a positive expected
count \(\mathbb E_\Pi V=\sum_k b_k\), which one complete truncating
round strictly decreases. Hence \(\Pi\widetilde P\ne\Pi\).

\paragraph{Supported full restart.}
At \(r=1\), no part of the current record is retained, so
\[
q_k^{\mathrm{full}}(\mathbf{y}):=q_{k,1}(\mathbf{y}\mid \mathbf{x})
=\prod_{t=1}^T g_{\alpha_k,t}(y_t\mid \mathbf{x}_0,\mathbf{y}_{<t})
\]
is independent of \(\mathbf{x}\). Suppose every rung selects this restart
with probability \(\omega_k(1)>0\), and
\(q_k^{\mathrm{full}}(\mathbf{y})>0\) for every \(\mathbf{y}\in\mathcal S_T\).
Define
\[
\varepsilon_k
=\omega_k(1)\min_{\mathbf{y}\in\mathcal S_T}
\frac{q_k^{\mathrm{full}}(\mathbf{y})}{\pi_k(\mathbf{y})},
\qquad
\varepsilon=\prod_{k=1}^K\varepsilon_k.
\]
The minimum ratio measures worst-case proposal coverage: it is the
largest constant for which the proposal assigns at least that
multiple of the target probability to every record. Multiplication
by \(\omega_k(1)\) accounts for how often a full restart is selected.
The proof below shows that, regardless of the current record, the
local transition contains a component of weight \(\varepsilon_k\)
distributed as \(\pi_k\). Independence makes \(\varepsilon\) the
corresponding weight of the product target in a joint local sweep.
Finiteness, positive support, and normalization imply
\(0<\varepsilon_k\le1\) and \(0<\varepsilon\le1\).

\begin{theorem}[Convergence of the PPT ensemble and returned rung]
\label{thm:fixed-horizon-convergence}
Under the supported-full-restart conditions above, \(\Pi\) is the
unique invariant law of \(P\). For any initializer \(\mu_0\) on
\(\mathcal S_T^K\), any rung \(k\), and \(N\ge1\),
\begin{equation}
\operatorname{TV}(\mu_N^{(k)},\pi_k)
\le\operatorname{TV}(\mu_N,\Pi)
\le(1-\varepsilon)^{N}
\operatorname{TV}(\mu_0,\Pi)
\xrightarrow[N\to\infty]{}0.
\label{eq:theory-fixed-horizon-convergence}
\end{equation}
\end{theorem}

\begin{proof}
The full-restart component of the local MH kernel satisfies
\[
\begin{aligned}
L_k(\mathbf{x},\mathbf{y})
&\ge\omega_k(1)q_k^{\mathrm{full}}(\mathbf{y})
\min\left\{1,
\frac{\pi_k(\mathbf{y})q_k^{\mathrm{full}}(\mathbf{x})}
{\pi_k(\mathbf{x})q_k^{\mathrm{full}}(\mathbf{y})}\right\}\\
&=\omega_k(1)\pi_k(\mathbf{y})
\min\left\{
\frac{q_k^{\mathrm{full}}(\mathbf{y})}{\pi_k(\mathbf{y})},
\frac{q_k^{\mathrm{full}}(\mathbf{x})}{\pi_k(\mathbf{x})}\right\}
\ge\varepsilon_k\pi_k(\mathbf{y}).
\end{aligned}
\]
For \(\mathbf{y}=\mathbf{x}\), self-proposals alone provide this lower bound;
rejections only add diagonal mass. Multiplying over rungs gives
\[
\mathcal R(\mathbf X,\mathbf Y)\ge\varepsilon\Pi(\mathbf Y)
\quad\text{for all }\mathbf X,\mathbf Y\in\mathcal S_T^K.
\]
If \(\varepsilon=1\), one local sweep has law \(\Pi\), and later
invariant updates preserve it. Otherwise define the residual kernel
\[
Q(\mathbf X,\mathbf Y)
=\frac{\mathcal R(\mathbf X,\mathbf Y)-\varepsilon\Pi(\mathbf Y)}
{1-\varepsilon}.
\]
The lower bound makes \(Q\) nonnegative, its rows sum to one, and
\(\Pi\mathcal R=\Pi\) implies \(\Pi Q=\Pi\). Consequently, it holds that
\[
\operatorname{TV}(\nu\mathcal R,\Pi)
=(1-\varepsilon)\operatorname{TV}(\nu Q,\Pi Q)
\le(1-\varepsilon)\operatorname{TV}(\nu,\Pi).
\]
This implies that each local refinement step contracts the error by a factor of at most
\(1-\varepsilon\). Additionally, since swaps preserve \(\Pi\) and cannot increase
TV, we conclude that the joint bound is proved for the entire algorithm.
Marginalization proves the rung bound. Convergence
from every initializer establishes uniqueness of the invariant law.
\end{proof}

\paragraph{Scope.}
It is noted that the supported full restarts are a sufficient condition for convergence.
Additionally, the coefficient $\varepsilon$ is a worst-case guarantee and can be very
small. Uniform restart selection alone contributes a factor $T^{-K}$ through
$\omega_k(1)=1/T$, and the minimum proposal-to-target ratio can be smaller
still. Theorem~\ref{thm:fixed-horizon-convergence} therefore establishes asymptotic
exactness; it does not certify a small error at a practical refinement
budget. At the terminal horizon, the law produced by the preceding block
extensions serves as $\mu_0$, and further fixed-horizon refinement converges
to the output target without requiring consistency between stage targets at
different horizons. The statement concerns refinement at a fixed terminal
horizon, not the number of block extensions.

\subsection{A finite illustration of target preservation and one-way truncation bias}
\label{sec:endpoint-bias-example-condensed}

Lastly, we present an example instantiating the lower bound of
Subsection~\ref{sec:fixed-horizon}. It contrasts a fixed-horizon
Metropolis--Hastings transition with the unbalanced current-length-capped
transition discussed above. The example makes both effects explicit on a common
4-state comparison space.
\\
\textbf{Four-state example.}
Take one rung with a fixed horizon \(T=2\), \(\alpha=2\), and vocabulary
\(\{a,e\}\), where \(e\) is terminal and
\(p_0(a\mid h)=p_0(e\mid h)=1/2\) at every open prefix. Let
\(X_1=(e,\bot)\), \(X_2=(a)\), \(X_3=(a,e)\), and
\(X_4=(a,a)\), with respective lengths \((1,1,2,2)\).
The valid fixed-horizon space is \(\mathcal S_2=\{X_1,X_3,X_4\}\);
\(X_2\) is prematurely censored and appears only in the enlarged
comparison space. Write \(p_{0,2}\) for the base record law at horizon
two, extended by zero on \(X_2\). Its probabilities and powered target,
obtained by normalizing \(p_{0,2}^2\) with \(Z=3/8\), are
\begin{equation}
(p_{0,2}(X_i))_{i=1}^4
=
\left(\frac12,0,\frac14,\frac14\right),
\qquad
\pi
=
\left(\frac23,0,\frac16,\frac16\right).
\label{eq:condensed-finite-summary}
\end{equation}
Here \(X_3\) terminates at the horizon, whereas \(X_4\) reaches the
horizon without a terminal token. Both are valid records; the shorter
unfinished record \(X_2\) has zero target mass, despite its nonzero
prefix likelihood. Symmetry gives the full-restart proposal
\(q^{\mathrm{full}}=p_{0,2}\), but \(q^{\mathrm{full}}\neq\pi\):
the proposal is normalized tokenwise, whereas \(\pi\) uses a single
sequence-level normalizer. Deterministic padding leaves these weights
unchanged.

The corrected update chooses a restart position uniformly from
\(\{1,2\}\) and regenerates through the common endpoint \(T\), padding
after EOS. Its full-restart component is an independence proposal with
acceptance probability
\(A_1(\mathbf{x},\mathbf{y})=1\wedge p_{0,2}(\mathbf{y})/p_{0,2}(\mathbf{x})\) on valid records.
The unbalanced update instead chooses the restart uniformly from
\(\{1,\ldots,\operatorname{len}(\mathbf{x})\}\), caps regeneration at \(\operatorname{len}(\mathbf{x})\), and applies
the token-score acceptance rule through the candidate endpoint without
checking reverse support. The two kernels are found to be
\begin{equation}
\begin{aligned}
P^{\mathrm{fix}}
&=
\begin{pmatrix}
7/8 & 0 & 1/16 & 1/16 \\
0   & 1 & 0    & 0    \\
1/4 & 0 & 3/8  & 3/8  \\
1/4 & 0 & 3/8  & 3/8
\end{pmatrix},\\[2mm]
\widetilde P
&=
\begin{pmatrix}
1/2 & 1/2 & 0   & 0   \\
1/2 & 1/2 & 0   & 0   \\
1/4 & 0   & 3/8 & 3/8 \\
1/4 & 0   & 3/8 & 3/8
\end{pmatrix}.
\end{aligned}
\label{eq:condensed-fixed-kernel}
\end{equation}
The self-loop at \(X_2\) in \(P^{\mathrm{fix}}\) only embeds the
corrected kernel in the comparison space; this zero-target state is
unreachable from \(\mathcal S_2\).

Then, we consider the laws
\(\nu_N=\delta_{X_4}(P^{\mathrm{fix}})^N\) and
\(\widetilde\nu_N=\delta_{X_4}\widetilde P^N\), with no block
extensions during refinement. Starting from \(X_3\) gives the same
TV curves: \(\pi(X_3)=\pi(X_4)\), and their transition rows coincide
in each kernel.
Direct calculation gives
\(\pi P^{\mathrm{fix}}=\pi\) and
\(\operatorname{TV}(\nu_N,\pi)\to0\). In contrast, \(\widetilde P\) neither
preserves \(\pi\) nor converges to it:
\(\widetilde\nu_N\to\tfrac12\delta_{X_1}+\tfrac12\delta_{X_2}\), and
\begin{equation}
\begin{aligned}
\pi\widetilde P
&=
\left(\frac5{12},\frac13,\frac18,\frac18\right)
\neq\pi,\\
\lim_{N\to\infty}
\operatorname{TV}(\widetilde\nu_N,\pi)
&=
\frac12\left(\frac16+\frac12+\frac16+\frac16\right)
=\frac12.
\end{aligned}
\label{eq:condensed-invariance-test}
\end{equation}
More explicitly, for \(N\ge1\),
\begin{equation}
\begin{aligned}
\operatorname{TV}(\nu_N,\pi)
&=\frac23\left(\frac58\right)^N,\\
\operatorname{TV}(\widetilde\nu_N,\pi)
&=\max\left\{
\frac16+\frac13\left(\frac34\right)^N,
\frac12-\frac23\left(\frac34\right)^N
\right\}.
\end{aligned}
\label{eq:condensed-tv-laws}
\end{equation}
At \(N=2\), the unbalanced TV error has decreased to \(17/48\),
but remains above the corrected error \(25/96\). The unbalanced error
reaches its minimum \(37/128\) at \(N=4\), then rises toward \(1/2\).
Finite-step improvement therefore does not establish target preservation.
Figure~\ref{fig:condensed-tv-summary} shows this transient behavior and
the subsequent asymptotic separation.

\begin{figure}[t]
    \centering
    \begin{minipage}[t]{0.59\linewidth}
        \vspace{0pt}
        \centering
        \textbf{(a) Selected TV values}

        \vspace{2mm}
        \small
        \setlength{\tabcolsep}{3.5pt}
        \begin{tabular}{c|cc}
            \hline
            \(N\) & Fixed horizon & Unbalanced \\
            \hline
            \(0\)  & \(0.83333\) & \(0.83333\) \\
            \(1\)  & \(0.41667\) & \(0.41667\) \\
            \(2\)  & \(0.26042\) & \(0.35417\) \\
            \(3\)  & \(0.16276\) & \(0.30729\) \\
            \(4\)  & \(0.10173\) & \(0.28906\) \\
            \(5\)  & \(0.06358\) & \(0.34180\) \\
            \(10\) & \(0.00606\) & \(0.46246\) \\
            \(\infty\) & \(0\)   & \(1/2\) \\
            \hline
        \end{tabular}
    \end{minipage}
    \hfill
    \begin{minipage}[t]{0.38\linewidth}
        \vspace{0pt}
        \centering
        \textbf{(b) Deviation from $\pi$}
        \vspace{1mm}
        \includegraphics[width=\linewidth]{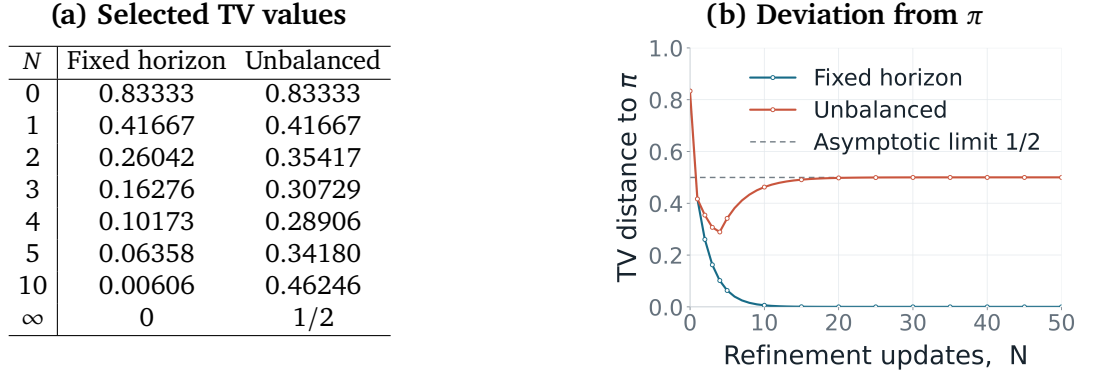}
    \end{minipage}
    \caption{
        Selected values and total-variation trajectories from \(X_4\)
        (identical TV values from \(X_3\)). The unbalanced curve uses the
        unbalanced kernel \(\widetilde P\). The dashed line is the attained
        asymptotic floor \(1/2\); TV can lie below it at finite \(N\).
    }
    \label{fig:condensed-tv-summary}
\end{figure}

\paragraph{Recovery of the structural lower bound.}
The example realizes the confinement mechanism of
Subsection~\ref{sec:truncation-lower-bound} with the closed set of
one-token records. From either longer record, the chain enters this
set in one step with probability \(1/4\), while the intended target
assigns mass \(1/3\) to its complement:
\[
\mathcal B=\mathcal B_1=\{X_1,X_2\},
\quad
\delta=\frac14,
\quad
b_1=\pi(\mathcal B^c)=\frac13.
\]
With \(K=n_{\mathrm{local}}=1\),
Eq.~\ref{eq:ppt-output-bias-lower} gives the finite-step lower bound
\(\bigl[\frac13-(\frac34)^N\bigr]_+\), and hence the asymptotic
floor \(1/3\). This length-only bound can be sharpened by using the
zero target mass of the prematurely censored record \(X_2\).
The second coordinate of \(\widetilde\nu_N\) gives, for \(N\ge1\),
\begin{equation}
\operatorname{TV}(\widetilde\nu_N,\pi)
\ge \widetilde\nu_N(X_2)
=\frac12-\frac23\left(\frac34\right)^N.
\label{eq:censored-mass-bound}
\end{equation}
For every \(N\ge4\), the mass on each valid record is at most its
target mass, so all excess mass lies on \(X_2\) and this bound is
an equality. It therefore recovers exactly the asymptotic floor
\(1/2\) attained by the unbalanced trajectory; the confinement bound
\(1/3\) remains a valid but weaker guarantee.
 
\paragraph{Takeaway.}
At a fixed refinement stage, a current-endpoint kernel---one whose
proposals cannot extend past the current record's length---can only
shorten a record or preserve its length. If it accepts shortening
moves with positive probability under the intended target~$\pi$,
mass flows irreversibly toward shorter records and $\pi$ cannot be
invariant; the early-stopping implementation's truncated-score
acceptance rule permits exactly this.
Fixed-horizon refinement removes this obstruction by defining
proposals on a common state space in which both directions have
positive probability, so the MH correction is non-degenerate and the
chain can move across record lengths.

\section{Power-Ladder Design}
\label{sec:appendix-ladder-design}

This section develops a principled design rule for the power ladder at a fixed
horizon $T$, following \citep{rathore2005optimal, syed2022nonreversible, kofke2002acceptance}. Recall that we denote the fixed-horizon state space with $\mathcal S_T$,
the sequence energy by $U(\mathbf{x}):=-\log p_0(\mathbf{x}\mid \mathbf{x}_0)$, and the the sequence-level power
target by
$\pi_\alpha(\mathbf{x}):=Z_\alpha(\mathbf{x}_0)^{-1}e^{-\alpha U(\mathbf{x})}$.
Fix the number of replicas $K$ and the endpoint powers
$\alpha_1$ and $\alpha_K$.  Every adjacent pair satisfies
$\alpha_k<\alpha_{k+1}$.  The analysis assumes equilibrium draws at the two
rungs involved in each exchange and studies the resulting expected swap
acceptance.

The design criterion below uses equilibrium overlap \citep{rathore2005optimal, DBLP:conf/aaai/000200LL23}.  Thus, each sequence from adjacent
temperatures is represented by independent draws from its two power targets.  This
criterion isolates the geometry of the ladder; observed acceptance in a
finite run also reflects warm-up, local mixing, and Monte Carlo variability. 

\subsection{Thermodynamic coordinate and exact swap overlap}
\label{subsec:thermodynamic-coordinate}

For $\alpha<\beta$, let $X_\alpha\sim\pi_\alpha$ and
$X_\beta\sim\pi_\beta$ independently.  The swap log ratio is
$G_{\alpha,\beta}:=(\beta-\alpha)
\{U(X_\beta)-U(X_\alpha)\}$, and the exact equilibrium acceptance is
$\bar a(\alpha,\beta):=
\mathbb E[1\wedge e^{G_{\alpha,\beta}}]$. Define
\(\sigma^2(s):=\operatorname{Var}_{\pi_s}[U(X)]\). The exponential-family
identity
\(
\frac{d}{ds}\mathbb E_{\pi_s}[U(X)]=-\sigma^2(s)
\)
and independence yield
\begin{equation}
    \mathbb E[G_{\alpha,\beta}]
    =-(\beta-\alpha)\int_\alpha^\beta \sigma^2(s)\,ds,
    \qquad
    \operatorname{Var}(G_{\alpha,\beta})
    =(\beta-\alpha)^2
      \{\sigma^2(\alpha)+\sigma^2(\beta)\}.
    \label{eq:swap-log-ratio-moments}
\end{equation}

To describe the local scale, let \(\sigma(\alpha)\) denote the energy
standard deviation under \(\pi_\alpha\). The corresponding thermodynamic
distance and total ladder length are
\[
d(\alpha,\beta)
=
\int_\alpha^\beta \sigma(s)\,ds,
\qquad
D_{\mathrm{tot}}=d(\alpha_1,\alpha_K).
\]
On the finite state space, the power family varies smoothly with \(\alpha\),
and its mean energy satisfies
\(
\frac{d}{d\alpha}\mathbb E_{\pi_\alpha}[U(X)]
=
-\sigma^2(\alpha)
\).
Together with the independence of \(X_\alpha\) and \(X_\beta\), this gives
the following local behavior when \(\beta=\alpha+\delta\):
\begin{equation}
\mathbb E[G_{\alpha,\alpha+\delta}]
=
-d(\alpha,\alpha+\delta)^2+O(\delta^3),
\qquad
\operatorname{Var}(G_{\alpha,\alpha+\delta})
=
2d(\alpha,\alpha+\delta)^2+O(\delta^3).
\label{eq:local-swap-moments}
\end{equation}
These expansions identify thermodynamic distance as the local scale of the
swap log ratio. They do not imply Gaussianity, so Gaussian shape is an
additional finite-gap approximation \citep{kofke2002acceptance}. Specifically, we use
\[
G_{\mathrm G}(\alpha,\beta)
\sim
\mathcal N\!\left(
-d(\alpha,\beta)^2,\,
2d(\alpha,\beta)^2
\right),
\]
and denote its predicted acceptance by
\(
a_{\mathrm G}(\alpha,\beta)
=
\mathbb E[1\wedge e^{G_{\mathrm G}(\alpha,\beta)}]
\).
All optimality claims below refer to this Gaussian surrogate.

\subsection{Max--min design and the unique equi-accepting ladder}
\label{subsec:maximin-ladder}

Adjacent exchanges form a path, so every trip between the endpoint rungs
must cross every interface. When interfaces are attempted equally often, the
smallest predicted adjacent acceptance is a natural bottleneck
proxy for end-to-end transport. For fixed \(K\), \(\alpha_1\), and
\(\alpha_K\), we consequently consider
\begin{equation}
\max_{\substack{
\alpha_2,\ldots,\alpha_{K-1}:\\
\alpha_1<\alpha_2<\cdots<\alpha_{K-1}<\alpha_K
}}
\min_{1\leq k<K}
a_{\mathrm G}(\alpha_k,\alpha_{k+1}).
\label{eq:maximin-ladder-problem}
\end{equation}
This criterion optimizes the weakest interface rather than the average
acceptance across the ladder.

\begin{lemma}[Optimal Gaussian equi-accepting ladder]
\label{thm:unique-maximin-ladder}
Suppose \(U\) takes at least two values on \(\mathcal S_T\). Under the
Gaussian surrogate, the predicted acceptance between \(\alpha<\beta\) is
\[
a_{\mathrm G}(\alpha,\beta)
=
A(d(\alpha,\beta)),
\qquad
A(d):=2\Phi\!\left(-\frac{d}{\sqrt2}\right),
\]
where \(\Phi\) is the standard normal cumulative distribution function.
The unique solution of Eq.~\ref{eq:maximin-ladder-problem} divides
the total thermodynamic length equally among the \(K-1\) adjacent
interfaces:
\begin{equation}
\int_{\alpha_k^{\mathrm{opt}}}^{\alpha_{k+1}^{\mathrm{opt}}}
\sigma(s)\,ds
=
\frac{D_{\mathrm{tot}}}{K-1},
\qquad
k=1,\ldots,K-1.
\label{eq:equal-thermodynamic-distance}
\end{equation}
Consequently, every interface has the common predicted acceptance
\(A(D_{\mathrm{tot}}/(K-1))\).
\end{lemma}

\begin{proof}
For \(d=d(\alpha,\beta)>0\) and
\(G_{\mathrm G}\sim\mathcal N(-d^2,2d^2)\), the Gaussian tail and its
exponential tilt give
\(\mathbb P(G_{\mathrm G}\geq0)
+\mathbb E[e^{G_{\mathrm G}}\mathbf 1\{G_{\mathrm G}<0\}]
=2\Phi(-d/\sqrt2)\).
This proves the stated acceptance formula, and \(A\) is strictly decreasing \citep{kofke2002acceptance}.

For any feasible ladder, let
\(d_k=d(\alpha_k,\alpha_{k+1})\). Positive support and nonconstant \(U\)
imply \(\sigma(\alpha)>0\), while additivity gives
\(\sum_{k=1}^{K-1}d_k=D_{\mathrm{tot}}\). Consequently,
\(\min_k a_{\mathrm G}(\alpha_k,\alpha_{k+1})
=A(\max_k d_k)\). Since
\(\max_k d_k\geq D_{\mathrm{tot}}/(K-1)\), with equality if and only if all
gaps equal their average, the equal-gap vector is uniquely optimal.

Finally, the map
\(F(\alpha)=d(\alpha_1,\alpha)/D_{\mathrm{tot}}\) is continuous and strictly
increasing, so the optimal gaps determine the unique rungs
\(\alpha_k^{\mathrm{opt}}
=F^{-1}((k-1)/(K-1))\), \(k=1,\ldots,K\).
\end{proof}

Thus equi-acceptance follows from the max--min objective: fixed total
thermodynamic length forces the optimal ladder to balance every bottleneck.
When $U$ is constant on $\mathcal S_T$, every power target is the same,
$\sigma(\alpha)=0$, and each swap is accepted.  Every feasible ladder is
then max--min optimal. Lemma~\ref{thm:unique-maximin-ladder} covers the
varying-energy regime in which rung placement affects transport \citep{rathore2005optimal, syed2022nonreversible}.

\begin{corollary}[Arithmetic and geometric ladder templates]
\label{cor:arithmetic-geometric-ladders}
\label{cor:geometric-power-ladder}
Fix $K$, $\alpha_1$, and $\alpha_K$, and apply the max--min construction of Lemma~\ref{thm:unique-maximin-ladder} to either of the following idealized
thermodynamic design curves on $[\alpha_1,\alpha_K]$.  Each curve yields a
unique surrogate ladder:
\begin{equation}
    \begin{aligned}
    \sigma(\alpha)=\sigma_0>0
    &\quad\Longrightarrow\quad
    \alpha_k^{\mathrm{opt}}
    =\alpha_k^{\mathrm{arith}}
    :=\alpha_1+
      \frac{k-1}{K-1}(\alpha_K-\alpha_1),
    \qquad k=1,\ldots,K,\\
    C(\alpha)=\alpha^2\sigma^2(\alpha)=c>0
    &\quad\Longrightarrow\quad
    \alpha_k^{\mathrm{opt}}
    =\alpha_k^{\mathrm{geom}}
    :=\alpha_1
      \left(\frac{\alpha_K}{\alpha_1}\right)^{(k-1)/(K-1)},
    \qquad k=1,\ldots,K.
    \end{aligned}
    \label{eq:arithmetic-geometric-ladders}
\end{equation}
\end{corollary}

\begin{proof}
When $\sigma(\alpha)=\sigma_0$, integration gives
$d(\alpha,\beta)=\sigma_0(\beta-\alpha)$.  Equal thermodynamic gaps are equal
raw-power gaps, and the fixed endpoints make each gap
$(\alpha_K-\alpha_1)/(K-1)$.  When $C(\alpha)=c$, one has
$\sigma(\alpha)=\sqrt c/\alpha$, so integration gives
$d(\alpha,\beta)=\sqrt c\log(\beta/\alpha)$.  Equal thermodynamic gaps are
then equal log-power gaps, and the fixed endpoints make each log gap
$\log(\alpha_K/\alpha_1)/(K-1)$.  These two substitutions yield
Equation~\ref{eq:arithmetic-geometric-ladders}; by Lemma~\ref{thm:unique-maximin-ladder}, every interface then has the common predicted acceptance $A(D_{\mathrm{tot}}/(K-1))$, i.e., $A\bigl(\sigma_0(\alpha_K-\alpha_1)/(K-1)\bigr)$ for the arithmetic template and $A\bigl(\sqrt{c}\,\log(\alpha_K/\alpha_1)/(K-1)\bigr)$ for the geometric one.
\end{proof}

These exact curves are idealized design models.  Approximate flatness of
$\sigma(\alpha)$ gives a near-arithmetic initialization, while approximate
constancy of $C(\alpha)$ gives a near-geometric initialization.  Pilot
calibration can refine either template toward equal thermodynamic gaps \citep{kofke2002acceptance}.

In a finite state space, exact constancy of $\sigma^2(\alpha)$ on an open
interval forces $\sigma^2(\alpha)=0$ and constant energy: analyticity extends
the constant value across the positive power domain, while concentration on
minimum-energy states sends the variance to zero as $\alpha\to\infty$.
Therefore, positive constant variance is an idealized local model, and
arithmetic spacing is a local empirical template.

\section{Chain Communication Protocol}
\label{sec:chain-communication}

Each adjacent-pair swap is itself an MH move on the joint target, so any composition
of swaps preserves $\Pi_m$ at every stage $m$, and hence $\Pi$: the communication
schedule cannot change what we sample, only how fast records travel across the
ladder. Hence, choosing one is an efficiency question, and its answer depends on
the regime.

Rungs are indexed $1,\ldots,K$ as in the main text, and edge $k\in\{1,\ldots,K-1\}$ joins rungs $k$ and $k+1$. An attempt on edge $k$ is accepted with probability $a_k\in(0,1]$; as in the remark of Section~\ref{sec:ladder-design}, acceptance decisions are taken to be independent with these fixed, history-independent probabilities, and we write
\[
r_k=1-a_k,\qquad \rho_k=\frac{r_k}{a_k},\qquad R=\sum_{k=1}^{K-1}\rho_k .
\]
The schedules are defined as follows.

\textbf{Ordered ADJ:} one sweep attempts edges $1,2,\ldots,K-1$
in order, applying each accepted exchange immediately. Observe the tagged
replica at complete-sweep boundaries. Its round trip starts at rung $1$,
visits rung $K$, and then returns to rung $1$.
Let $T_\mathrm{ADJ}$ count sweeps, referring to one tagged replica's trip, 
rather than completion of a trip by
every replica. 
We proceed to derive the expected iteration count for a round trip $\E T_\mathrm{ADJ}$.

\emph{Upward passage.}
Follow the first upward crossing of each edge $k$.
Successful crossings can cascade through the entire ladder within one
sweep, so the passage takes one sweep plus delays caused by rejections.
After a rejection at edge $k$, the tag stays in the prefix
$\{1,\ldots,k\}$ until its next attempt at that edge. Observe the tag
just before edge $k$ is scheduled in each sweep. Until the next attempt,
its motion is the ordered scan confined to this prefix. Starting from
$k$, the mean time to return to $k$ and retry is also $k$ sweeps.
The number of rejections before the successful crossing is geometric,
with mean $\rho_k=r_k/a_k$. Independence and the Markov property make
the expected delay at this edge $k\rho_k$. Summing gives
\begin{equation}\label{eq:ord-up}
H^\mathrm{ADJ}_{1\to K}=1+\sum_{k=1}^{K-1}k\,\rho_k.
\end{equation}

\emph{Downward passage.}
A successful move from $k+1$ to $k$ occurs after edge $k-1$ has already
been processed. Thus the tag can move down by at most one rung per sweep,
so the passage takes $K-1$ sweeps plus rejection delays. After rejecting
edge $k$ from rung $k+1$, the tag makes an excursion in the suffix
$\{k+1,\ldots,K\}$ before retrying. Observing just before edge $k$ is
scheduled, the return-time fact gives a mean retry delay of $K-k$ sweeps.
There are on average $\rho_k$ rejections before the successful downward
crossing, hence
\begin{equation}\label{eq:ord-down}
H^\mathrm{ADJ}_{K\to1}=(K-1)+\sum_{k=1}^{K-1}(K-k)\rho_k.
\end{equation}
An endpoint reached during a sweep remains occupied until its end,
so these passage counts agree with the stated sweep-boundary convention.
Adding the two expectations gives 
\begin{equation}
    \E T_\mathrm{ADJ}=K(1+R).
\end{equation}

Lastly, we can compare the expected iteration count above with the
corresponding expected count for the asymptotically optimal schedule 
deterministic even--odd (DEO) introduced by \citep{okabe2001replica}.
More specifically in DEO, one layer attempts all even or all odd edges,
alternating the two matchings deterministically. In the lifted state
$(k,\varepsilon)$, the sign $\varepsilon\in\{-1,+1\}$ is the next
proposed direction. Acceptance preserves the sign; rejection reverses it.
An outward direction at an endpoint reverses in one idle layer. We count
the round trip, including both turnaround
layers. This is the recurring endpoint-arrival convention.

The corresponding expected count is 
\begin{equation}
    \E T_\mathrm{DEO}=2K(1+R).
\end{equation}

Subsequently, assume $d_{\mathrm{ADJ}}$ and $d_{\mathrm{DEO}}$ denote the 
iteration durations, including refinement and swap communication. When
the additional communication cost of ordered ADJ is negligible
relative to $d_{\mathrm{DEO}}$, the expected round-trip elapsed times satisfy
\[
    \frac{\mathbb{E}C_{\mathrm{ADJ}}}
         {\mathbb{E}C_{\mathrm{DEO}}}
    = \frac{d_\mathrm{ADJ}}{2d_{\mathrm{DEO}}}
    = \frac12+
      \frac{d_{\mathrm{ADJ}}-d_{\mathrm{DEO}}}
           {2d_{\mathrm{DEO}}}
    \approx \frac12.
\]

\begin{remark}
DEO becomes preferable when the additional sequential communication
outweighs this iteration-count advantage, namely when
$d_{\mathrm{ADJ}}>2d_{\mathrm{DEO}}$. In particular, as $K\to\infty$,
the quadratic scaling $\Theta(K^2(1+R))$ of the ordered sweep dominates.
Under these timing assumptions, DEO is asymptotically faster
as the number of replicas grows.
\end{remark}

\section{Extended Related Work}\label{sec:related}

\noindent\textbf{Distribution sharpening for LLMs.}
A major goal of LLM post-training is to shift probability mass toward higher-quality responses. RLHF optimizes against learned human-preference rewards~\citep{ouyang2022training}, while RLVR methods such as GRPO optimize against verifiable task rewards and have improved mathematical and coding performance~\citep{shao2024deepseekmath,guo2025deepseekr1,hu2025openreasonerzero,lambert2024tulu3}. However, post-training is expensive, reward-dependent, and can reduce output diversity by reinforcing already likely solution modes~\citep{he2025rewarding,NEURIPS2025_537d5aa7,chen2025posttraining}. This motivates training-free inference-time methods that sharpen or reshape the base distribution without updating model weights.

\noindent\textbf{Inference-time scaling.}
Best-of-$N$ and reward-reranking improve generation by sampling multiple completions and selecting a high-scoring one~\citep{wang2025effect,huang2025bestofn}. Temperature and truncation methods, including low-temperature, modify local token probabilities to balance quality and diversity~\citep{meister2023locally,troshin2025control,du2025optimizing}. These approaches are effective but are generally heuristic: they do not sample from a specified sequence-level target distribution. 

\noindent\textbf{Power distribution sampling.}
Power distribution sampling replaces these heuristics with an explicit sequence-level target, the power distribution $\pi_\alpha\propto p_0^{\alpha}$ with $\alpha>1$ (Eq.~\ref{eq:pi-alpha}), which sharpens the base model over whole completions rather than token by token. \citet{karan2025reasoning} introduced Power Sampling, which targets $\pi_\alpha$ with a Metropolis--Hastings chain whose proposals regenerate a uniformly chosen suffix under the tokenwise-powered proposal $g_\alpha$ (Eq.~\ref{eq:proposal}), applied blockwise over progressively longer prefixes. Without training, rewards, or verifiers, it nearly matches and in some cases exceeds GRPO on MATH500, HumanEval, and GPQA, but repeated suffix regeneration makes it expensive at inference time. Two follow-ups reduce this cost. \citet{ji2026scalable} show that $\pi_\alpha$ can be approximated autoregressively by the low-temperature distribution rescaled with token-level factors that capture the quality of future continuations; they estimate these factors with Monte Carlo rollouts and a jackknife bias correction, which removes the MCMC loop and reduces latency by more than an order of magnitude. PowerSMC~\citep{DBLP:arxiv/power-smc} instead targets $\pi_\alpha$ with sequential Monte Carlo: a small set of particles is advanced in parallel under $g_\alpha$, which they show is the unique prefix-only proposal minimizing the incremental weight variance, with token-by-token importance-weight correction, resampling when needed, and an exponent-bridging schedule that stabilizes the particles. It matches or exceeds MH-based Power Sampling on MATH500 at a fraction of its latency. \method builds on the local MH refinement of Power Sampling, and we compare against both Power Sampling and PowerSMC in our experiments.

\noindent\textbf{Monte Carlo methods for LLM inference.}
MCMC and SMC provide principled alternatives to heuristic decoding for unnormalized sequence-level targets. MCMC constructs a Markov chain with the desired stationary distribution, and has been used for energy-based controllable text generation, blockwise text resampling, quality-aware machine translation, and constrained LM sampling~\citep{mireshghallah2022mix,forristal2023block,faria2024quest,anaya2026constrained}. SMC instead maintains weighted particles across a sequence of intermediate targets and has recently been applied to twisted SMC for language-model inference, mathematical reasoning, and reward-guided decoding~\citep{zhao2024probabilisticinference,feng2024step,markovic2026sampling}. 

\noindent\textbf{Parallel tempering.}
Parallel tempering is designed to improve MCMC mixing by running multiple tempered replicas and swapping their states with a Metropolis correction. It originated in spin-glass simulation~\citep{swendsen1986replica}, was generalized to statistical MCMC~\citep{geyer1991markov}, and became standard in exchange Monte Carlo and molecular simulation~\citep{hukushima1996exchange,sugita1999replica,earl2005parallel}. Later work tunes the temperature ladder~\citep{kone2005efficiency}, optimizes the annealing path~\citep{syed2021parallel}, and uses non-reversible replica dynamics~\citep{okabe2001replica,syed2022nonreversible}. These methods exploit high-temperature replicas for exploration and low-temperature replicas for target concentration. Existing LLM samplers rarely use this exchange mechanism: they typically rely on a single MH chain~\citep{karan2025reasoning}, independent best-of-$N$ samples~\citep{huang2025bestofn}, or SMC particles~\citep{DBLP:arxiv/power-smc}. \method adapts parallel tempering to LLM sampling by transferring states from exploratory replicas to sharpened targets, with the aim of improving finite-budget mixing while exploiting parallel execution.

Recently, \citet{he2026recipes} developed replica-exchange sampling for LLMs using intermediate targets at a common sharpening temperature indexed by prefix length, with exchange proposals that extend and truncate token chunks. Instead, \method couples replicas at different sharpening powers over a common horizon at each generation stage, exchanging entire sequence records using cached likelihoods without additional model evaluations. This construction explicitly couples exploration under flatter distributions with concentration under sharper targets over the same completion space. 

\section{Experimental Settings}\label{sec:exp-settings}

\subsection{Choice of Baselines}
\begin{itemize}[leftmargin=*]
    \item \textbf{Standard Decoding.} We sample directly from the autoregressive model $p_0$ using standard ancestral decoding at temperature one, corresponding to the unsharpened target $\alpha=1$. This baseline measures the raw capability of the underlying model and serves as the reference against which all sharpening gains are reported.
    \item \textbf{Low Temperature Decoding.} We decode from the autoregressive model at a reduced sampling temperature $\tau=1/\alpha<1$, the standard heuristic for concentrating probability on high-likelihood tokens. Because temperature sharpens each next-token conditional independently, it does not reproduce the sequence-level power target $\pi_{\alpha}\propto p_0^{\alpha}$ (Equation~\ref{eq:theory-power-conditional}).
    
    \item \textbf{Power Sampling (Blockwise MH Sampler)}~\citep{karan2025reasoning}\textbf{.} For the \emph{Power Sampling} results of Table~\ref{tab:main-results}, we use the authors' released implementation without modification. In particular it retains the early stopping analyzed in Section~\ref{sec:fixed-horizon}, where each proposal is generated and scored only through the current realized length. 
\item \textbf{Power Sequential Monte Carlo}: A particle-based alternative that approximates the same tempered target by maintaining a population of partial sequences and periodically resampling them in proportion to their power-reweighted importance weights. It represents the sequential Monte Carlo counterpart to our parallel tempering scheme, allowing us to contrast population resampling against exchange-based transport.

\item \textbf{GRPO}: Group Relative Policy Optimization~\citep{shao2024deepseekmath} is a reinforcement-learning method that fine-tunes the model weights to concentrate probability mass on high-reward trajectories. Unlike the inference-time samplers above, GRPO modifies the model itself and requires a verifiable reward signal. we include it as a strong training-based reference point to assess how far a purely inference-time sampler can close the gap to a method that pays the additional cost of post-training.
\end{itemize}

For Power Sampling~\citep{karan2025reasoning} and PowerSMC, we based our experiments on the respective official implementations, left their sampling kernels unmodified, and tuned their method-specific hyperparameters independently for each model to obtain the strongest performance we could achieve under the common evaluation protocol.

\subsection{Choice of Datasets}

\begin{itemize}[leftmargin=*, nosep]
    \item {\textbf{MATH500}}: The {MATH} dataset~\citep{hendrycks2021math} consists of competition math problems spanning seven categories including geometry, number theory, and precalculus. There are 12500 problems total, with 7500 training problems and 5000 test problems. {MATH500} is a fixed subset of 500 problems from the test set~\citep{lightman2024lets}.

    \item \textbf{HumanEval}: HumanEval is a set of 164 handwritten programming problems covering algorithms, reasoning, mathematics, and language comprehension \cite{chen2021evaluatingllmcode}. Each problem has an average of 7.7 associated unit tests, where solving the problem corresponds to passing all unit tests.

    \item {\textbf{GPQA}}: GPQA \cite{rein2024gpqa} is a dataset of multiple-choice science questions (physics, chemistry, and biology) which require advanced reasoning skills to solve. We use the GPQA Diamond subset for evaluation, which consists of 198 questions which represent the highest quality subset of the GPQA dataset.

    \item {\textbf{GSM8K}}: GSM8K is a dataset of grade-school math word problems that require multi-step arithmetic reasoning to solve~\citep{gsm8k_dataset}. Each problem is paired with a natural-language solution and a single numeric final answer, and correctness is scored by exact match on that answer. We evaluate on the standard 1319-problem test split.
    \item {\textbf{AIME 24\&25}}: The AIME 24\&25 sets collect the problems from the American Invitational Mathematics Examination for those years, a high-difficulty olympiad-style competition whose questions each have an integer answer between 0 and 999~\citep{aime24, aime25}.
    We evaluate on the combined $60$-problem set, both examinations (I and II) from $2024$ and $2025$ respectively. This benchmark probes performance on the hardest, most reasoning-intensive mathematics in our suite, where exploration of distinct solution paths matters most.
    \item {\textbf{LCB v5}}: LiveCodeBench is a contamination-aware coding benchmark that continuously collects problems from competitive-programming contests \citep{DBLP:conf/iclr/JainHGLYZWSSS25}. 
\end{itemize}

\subsection{Evaluation Metrics}

\paragraph{Task accuracy.}
Our primary metric is strict pass@1 accuracy: for each benchmark item, every method returns exactly one completion, and we report the percentage of items solved. No method uses reward- or verifier-based reranking or oracle filtering. For GSM8K and AIME 24\&25, we extract the final numeric answer and require an exact match after standard normalization. GPQA is scored by exact match on the selected multiple-choice option. MATH500 is scored with the PRM800K/Hendrycks-MATH equivalence grader.

\paragraph{Code correctness.}
For HumanEval, and LiveCodeBench~v5, we report execution-based pass@1. A completion is correct only if it compiles or executes successfully and passes every associated unit test; invalid programs, runtime errors, and timeouts receive zero credit. 

\paragraph{Diagnostic and efficiency metrics.}
To characterize the generated reasoning traces, we additionally report the average token log-likelihood under the base model,
$T^{-1}\sum_{t=1}^{T}\log p_0(x_t\mid x_{<t})$. Confidence is measured as the mean negative next-token entropy,
\[
C(x)=\frac{1}{T}\sum_{t=1}^{T}\sum_{v\in\mathcal V}
p_0(v\mid x_{<t})\log p_0(v\mid x_{<t}),
\]
with values closer to zero indicating greater confidence; we first average over generated positions within each completion and then average the resulting completion-level scores across benchmark items.

\subsection{Configuration}

\label{sec:hyperparameter-configuration}

\paragraph{Models and prompting.}
The evaluated checkpoints range from 3.8B to 9B parameters and include
base, instruction-tuned, and GRPO-trained models from the Qwen2.5, Qwen3,
Qwen3.5, Phi-3.5, and Tulu-3 families. No additional training is performed as part
of our evaluation. The Qwen3 models use their tokenizer-provided chat
templates, whereas the Qwen2.5 models use raw prompts. Experiments were run
on NVIDIA H100 hardware. The requested completion budget is
capped per item by the checkpoint context length minus the prompt length
and a 64-token safety margin. The Qwen2.5-Math experiments use
$T=3072$. For Qwen3, Table \ref{tab:ppt-hyperparams-qwen3-4b} shows: GSM8K use $T=3072$, HumanEval use $T=4096$, MATH500 and LCB v5 experiments use
$T=8192$,  Qwen3 GPQA experiments use $T=16384$, and  Qwen3 AIME 24\&25 experiments use $T=32768$.
Lastly, for Qwen3.5 GPQA and LCB v5 use $T=131072$, and AIME 24\&25 use $T=65536$.


\paragraph{PowerSMC implementation.}
PowerSMC maintains normalized particle weights $\bar w_i$ and computes
\[
    \mathrm{ESS}
    =
    \frac{1}{\sum_{i=1}^{N}\bar w_i^2}.
\]
ESS is checked at the configured token interval and at the final horizon.
When it falls below $0.5N$, the implementation performs systematic
resampling and resets the particle weights. At the end of generation, one
particle is sampled according to the final normalized weights. Particles
are internal sampler states and are not counted as additional evaluation
completions. Lastly, PowerSMC uses a per-particle EOS handling.

\begin{table*}[t]
\centering
\caption{PPT hyperparameters for the
\texttt{Qwen3-4B} and \texttt{Qwen3-8B} evaluations.
$B$ is the generation block size, and $N_{\mathrm{MCMC}}$ is the number of local refinement steps per block.}
\label{tab:ppt-hyperparams-qwen3-4b}
\setlength{\tabcolsep}{3.5pt}
\begin{tabular}{lccccc}
\toprule
Benchmark
& $K$
& $\{\alpha_k\}$
& $T$
& $B$
& $N_{\mathrm{MCMC}}$\\
\midrule
MATH500
& 4
& $\{2.35,2.5,2.7,2.9\}$
& 8192
&4096
& 10\\

AIME 24\&25
& 4
& $\{1.8,1.95,2.15,2.4\}$
& 32768
& 8192
& 10\\

GPQA-Diamond
& 3
& $\{1.15,1.3,1.5\}$
& 16384
& 4096
& 6
\\

LiveCodeBench v5
& 4
& $\{1.8,2.0,2.2,2.4\}$
& 8192
& 2048
& 8\\

HumanEval
& 3
& $\{1.5, 1.6, 1.8\}$
& 4096
& 2048
& 20\\

GSM8K 
& 6
& $\{2.0, 2.3, 2.6, 3.0, 3.5, 4.0\}$
& 3072
& 192
& 10\\
\bottomrule
\end{tabular}
\end{table*}

\begin{table*}[t]
\centering
\caption{PPT hyperparameters for the
\texttt{Qwen3.5-9B} evaluations.
$B$ is the generation block size, and $N_{\mathrm{MCMC}}$ is the number of local refinement steps per block.}
\label{tab:ppt-hyperparams-qwen35-9b}
\setlength{\tabcolsep}{3.5pt}
\begin{tabular}{lccccc}
\toprule
Benchmark
& $K$
& $\{\alpha_k\}$
& $T$
& $B$
& $N_{\mathrm{MCMC}}$\\
\midrule
GPQA-Diamond
& 3
& $\{1.05, \,1.10, \,1.25\}$
& 131{,}072
& 131{,}072
& 10\\

AIME 24\&25
& 3
& $\{1.25,\,1.4,\,1.6\}$
& 65{,}536
& 65{,}536
& 5\\

LiveCodeBench v5
& 3
& $\{1.1,\,1.25,\,1.4\}$
& 131{,}072
& 131{,}072
& 5\\
\bottomrule
\end{tabular}
\end{table*}
Across all configurations, replica $K$ is the returned replica and has target
power $\alpha^\star$. Local proposals draw the restart position uniformly
from $\{1,\ldots,T_m\}$ and regenerate the suffix using temperature
$\tau_k=1/\alpha_k$, top-$p=1$, disabled top-$k$, and min-$p=0$. We apply one
ordered adjacent sweep, $(1,2),(2,3),\ldots,(K-1,K)$, after every synchronized local-MH round.

\paragraph{Common Comparison Protocol}
\label{apx:common-protocol}

Two entries are compared, in Table~\ref{tab:main-results} or in
Appendix~\ref{apx:extra-exp}, only when they satisfy the following
conditions:
\begin{itemize}[leftmargin=*, nosep]
    \item \textbf{Grading.} Public or official graders throughout.
    \item \textbf{Prompt.} Every baseline is compared with \method{} under an identical prompt.
    \item \textbf{Completion cap.} Identical horizon for all baselines.
\end{itemize}

\section{Extended Experimental Results} \label{apx:extra-exp}

\subsection{Extended Main Results}
\newcommand{\best}[1]{\textbf{#1}}
\newcommand{\second}[1]{\underline{#1}}
\providecommand{\sd}[1]{{\scriptsize$\pm$#1}}

\begin{table*}[!t]
\centering
\caption{Single-sample performance (\%) across reasoning, coding, and knowledge benchmarks. \textbf{Bold} indicates the best result within each base-model group. Scores are averaged over 5 random seeds and reported as mean $\pm$ standard deviation. Dashes '\textemdash{}' indicate that no GRPO result is available for Phi 3.5 instruct.}
\label{tab:main-results-extended}
\begin{adjustbox}{width=\textwidth}
\begin{tabular}{l|cccccc}
\toprule
\textbf{Method} &
\textbf{MATH 500} &
\textbf{GPQA} &
\textbf{Human Eval} &
\textbf{GSM8K} &
\textbf{AIME 24\&25} &
\textbf{LCB v5} \\
\midrule

\multicolumn{7}{l}{{\textbf{Qwen 2.5 MATH}}} \\
\hspace{12pt} Standard
& 51.6 \sd{1.8} & 31.6 \sd{1.1} & 30.5 \sd{1.8} & 71.4 \sd{0.9} & 4.3 \sd{0.9} & 2.5 \sd{0.9} \\
\hspace{12pt} Lower Temperature
& 72.7 \sd{0.7} & 35.9 \sd{2.3} & 53.0 \sd{2.2} & 85.1 \sd{1.2} & 11.0 \sd{1.5} & 7.6 \sd{2.1} \\
\hspace{12pt} Power Sampling
& 76.7 \sd{0.4} & 35.9 \sd{1.8} & 56.7 \sd{0.4} & 86.7 \sd{0.3} & 15.7 \sd{1.5} & 8.9 \sd{1.2} \\
\hspace{12pt} PowerSMC
& 78.0 \sd{0.7} & 32.8 \sd{1.4} & 59.1 \sd{2.6} & 85.4 \sd{1.0} & 13.3 \sd{1.2} & 10.7 \sd{1.0} \\
\hspace{12pt} GRPO
& 77.3 \sd{0.6} & 36.5 \sd{2.7} & 53.7 \sd{1.1} & 86.4 \sd{0.4} & 13.3 \sd{1.7} & 2.8 \sd{1.1} \\
\rowcolor{pptrow} \hspace{12pt}  \textbf{\method (Ours)}
& \textbf{79.6} \sd{0.8} & \textbf{37.6} \sd{1.1} & \textbf{60.6} \sd{1.3}
& \textbf{91.4} \sd{1.1} & \textbf{17.7} \sd{1.5} & \textbf{12.5} \sd{1.9} \\

\midrule

\multicolumn{7}{l}{ {\textbf{Qwen 2.5}}} \\
\hspace{12pt} Standard
& 43.0 \sd{3.0} & 28.6 \sd{2.0} & 20.7 \sd{0.9} & 59.2 \sd{1.1} &  4.3 \sd{1.9} & 2.5 \sd{0.3} \\
\hspace{12pt} Lower Temperature
& 63.8 \sd{2.2} & 32.8 \sd{2.1} & 51.8 \sd{0.7} & 81.0 \sd{0.5} &  6.7 \sd{2.0} & 15.1 \sd{1.8} \\
\hspace{12pt} Power Sampling
& 69.6 \sd{2.1} & 33.2 \sd{2.9} & 58.7 \sd{2.5} & 86.1 \sd{0.5} &  9.0 \sd{1.5} & 6.6 \sd{0.8} \\
\hspace{12pt} PowerSMC
& 70.3 \sd{0.7} & 34.0 \sd{0.6} & 59.4 \sd{0.8} & 86.3 \sd{0.4} &  12.7 \sd{1.9} & 15.0 \sd{0.7} \\
\hspace{12pt} GRPO
& \textbf{74.0} \sd{1.1} & 33.5 \sd{2.5} & 57.3 \sd{0.9} & 84.2 \sd{0.4} & 13.3 \sd{1.2} & \textbf{16.9} \sd{1.2} \\
\rowcolor{pptrow} \hspace{12pt} \textbf{\method (Ours)}
& \textbf{74.0} \sd{1.5} & \textbf{34.9} \sd{0.7} & \textbf{61.6} \sd{0.9} & \textbf{88.1} \sd{0.5} & 
 \textbf{15.7} \sd{0.9} & 9.3 \sd{0.5} \\

\midrule

\multicolumn{7}{l}{ {\textbf{Phi 3.5 instruct}}}\\
\hspace{12pt} Standard
& 37.9 \sd{1.4} & 31.6 \sd{1.4} & 37.9 \sd{1.1} & 80.4 \sd{0.5} &  2.3 \sd{0.9} & 7.2 \sd{1.3} \\
\hspace{12pt} Lower Temperature
& 41.1 \sd{0.8} & 31.3 \sd{0.4} & 61.0 \sd{1.2} & 82.3 \sd{0.4} &  4.3 \sd{1.5} & 8.1 \sd{1.8} \\
\hspace{12pt} Power Sampling
& 38.2 \sd{1.2} & 29.8 \sd{2.6} & 67.0 \sd{0.8} & 81.9 \sd{0.7} &  \textbf{6.7} \sd{1.2} & 11.7 \sd{1.9} \\
\hspace{12pt} PowerSMC
& 45.1 \sd{0.8} & 31.8 \sd{1.1} & 63.4 \sd{0.4} & 86.7 \sd{0.6} &  4.3 \sd{0.9} & 8.4 \sd{2.2} \\
\hspace{12pt} GRPO
& \textemdash{} & \textemdash{} & \textemdash{} & \textemdash{} & \textemdash{} & \textemdash{} \\
\rowcolor{pptrow} \hspace{12pt} \textbf{\method (Ours)}
& \textbf{48.4} \sd{1.1} & \textbf{33.4} \sd{0.9} & \textbf{68.3} \sd{0.7}
& \textbf{88.6} \sd{0.4} &  {4.3} \sd{1.5} & \textbf{12.0} \sd{2.2} \\
\midrule

\multicolumn{7}{l}{ {\textbf{Tulu}}} \\
\hspace{12pt} Standard
& 44.9 \sd{1.0} & 26.8 \sd{0.5} & 59.8 \sd{1.8} & 85.4 \sd{0.3} &  2.3 \sd{1.5} & 9.8 \sd{0.9} \\
\hspace{12pt} Lower Temperature
& 45.2 \sd{1.1} & 26.3 \sd{0.5} & 63.4 \sd{1.5} & 87.3 \sd{0.2} &  2.3 \sd{1.5} & 10.0 \sd{1.3} \\
\hspace{12pt} Power Sampling
& 49.9 \sd{0.9} & 32.3 \sd{2.2} & 59.8 \sd{0.9} & 87.8 \sd{0.1} &  4.3 \sd{0.9} & 11.1 \sd{0.9} \\
\hspace{12pt} PowerSMC
& 50.3 \sd{0.6} & 31.8 \sd{0.6} & 60.2 \sd{0.9} & 87.7 \sd{0.2} &  4.3 \sd{0.9} & 12.2 \sd{0.7} \\
\hspace{12pt} GRPO
& 47.0 \sd{0.8} & 24.7 \sd{0.4} & 61.3 \sd{1.1} & 87.6 \sd{0.1} &  4.3 \sd{1.5} & 11.0 \sd{0.5} \\
\rowcolor{pptrow} \hspace{12pt} \textbf{\method (Ours)}
& \textbf{51.8} \sd{0.6} & \textbf{35.9} \sd{0.7} & \textbf{64.0} \sd{1.6}
& \textbf{89.7} \sd{0.3} &  \textbf{6.7} \sd{3.1} & \textbf{12.6} \sd{0.6} \\

\midrule
 {\textbf{Qwen3-4B}}
 \\
\hspace{12pt} Standard
& 83.3 \sd{1.6} & 51.0 \sd{1.3} & 85.4 \sd{2.6} & 94.6 \sd{0.1} &  70.7 \sd{1.9} & 45.6 \sd{0.4} \\
\hspace{12pt} Lower Temperature
& 85.1 \sd{0.4} & 48.5 \sd{1.2} & 86.6 \sd{1.8} & 94.5 \sd{0.2} & 72.0 \sd{1.8} & 44.8 \sd{1.5} \\
\hspace{12pt} Power Sampling
& 83.5 \sd{2.5} & 51.0 \sd{0.7} & 90.2 \sd{1.5} & 94.2 \sd{0.1} &  73.3 \sd{1.2} & 55.1 \sd{1.7} \\
\hspace{12pt} PowerSMC
& 83.2 \sd{1.2} & 47.5 \sd{1.4} & 87.9 \sd{0.9} & 94.3 \sd{0.1} &  61.7 \sd{1.2} & 52.3 \sd{1.5} \\
\hspace{12pt} GRPO
& 85.6 \sd{0.6} & 50.0 \sd{0.6} & 91.6 \sd{0.7} & 94.6 \sd{0.1} &  73.3 \sd{1.7} & 52.7 \sd{1.8} \\
\rowcolor{pptrow} \hspace{12pt} \textbf{\method (Ours)}
& \textbf{86.0} \sd{0.7} & \textbf{53.6} \sd{0.9} & \textbf{93.9} \sd{2.3}
& \textbf{94.8} \sd{0.2} &  \textbf{78.3} \sd{2.4} & \textbf{57.9} \sd{1.3} \\

\midrule

 {\textbf{Qwen3-8B}}
\\
\hspace{12pt} Standard
& 83.9 \sd{1.2} & 55.8 \sd{1.7} & 84.9 \sd{1.2} & 94.6 \sd{0.3} &  72.7 \sd{3.8} & 48.1 \sd{1.7} \\
\hspace{12pt} Lower Temperature
& 84.6 \sd{0.4} & 54.5 \sd{2.1} & 87.2 \sd{0.6} & 95.0 \sd{0.2} &  73.0 \sd{4.6} & 48.4 \sd{2.5} \\
\hspace{12pt} Power Sampling
& 82.2 \sd{1.8} & 57.6 \sd{1.6} & 90.7 \sd{1.6} & 95.2 \sd{0.5} &  73.3 \sd{2.4} & 53.5 \sd{1.2} \\
\hspace{12pt} PowerSMC
& 84.1 \sd{0.6} & 58.0 \sd{0.6} & 90.2 \sd{0.4} & \textbf{95.5} \sd{0.1} &  65.0 \sd{1.2} & 53.6 \sd{1.0} \\
\hspace{12pt} GRPO
& 86.4 \sd{0.6} & 53.0 \sd{0.7} & 93.3 \sd{1.2} & 94.6 \sd{0.4} &  74.0 \sd{0.9} & 54.1 \sd{2.3} \\
\rowcolor{pptrow} \hspace{12pt} \textbf{\method (Ours)}
& \textbf{88.0} \sd{0.6} & \textbf{60.1} \sd{0.8} & \textbf{94.9} \sd{1.5}
& \textbf{95.5} \sd{0.2} &  \textbf{78.3} \sd{2.6} & \textbf{58.4} \sd{1.3} \\

\bottomrule
\end{tabular}
\end{adjustbox}
\end{table*}

Table~\ref{tab:main-results-extended} shows that \method{} is the most consistent method across model families and task types, attaining the best or tied-best performance in the large majority of settings. Its advantage is clearest for the Qwen3 models on the hardest tasks: the largest margin over the strongest baseline is on AIME 24\&25 for Qwen3-4B (5.0 points) and on LCB v5 for Qwen3-8B (4.3 points). Results from other model families are more mixed. On tasks where the model is already near saturation, \method{} remains competitive. This pattern is consistent with the central intuition behind parallel tempering: exploratory replicas can search across different reasoning modes, while exchange lets the output rung inherit promising trajectories instead of having to discover them on its own. Because the main comparison changes more than one component relative to Power Sampling, the table establishes the breadth of the improvement; the matched ablations below isolate the effects of fixed-horizon refinement and replica exchange.

\paragraph{Acceptance Rates}
Table~\ref{tab:swap-acceptance} reports the mean acceptance rate of swap
proposals between adjacent rungs, averaged over all exchange attempts in
the run, across six benchmarks and six models.  Two regularities stand
out.  First, acceptance is governed by the model and by the task.
Second, the Qwen3 models accept far fewer swaps $(0.18-0.37)$ than the other
models $(0.35-0.58)$.  Both regimes nonetheless sit
in the range classically recommended for parallel tempering (roughly
$0.15-0.6$), so the ladders are
communicating rather than idling; a denser ladder for the newer and more capable
models would raise acceptance further at the cost of more replicas, and we
leave that trade-off to the ladder-design discussion.

\begin{table}[t]
  \centering
  \small
  \setlength{\tabcolsep}{4.5pt}
  \caption{Mean swap acceptance of the replica ladder.}
  \label{tab:swap-acceptance}
  \begin{adjustbox}{max width=\linewidth}
  \begin{tabular}{l cccccc}
    \toprule
    Dataset & Qwen3-8B & Qwen3-4B & Qwen 2.5 MATH & Qwen 2.5 & Tulu & Phi 3.5 instruct \\
    \midrule
    MATH500       & 0.199 & 0.242 & 0.446          & 0.500          & 0.487          & {0.373} \\
    AIME 24\&25     & 0.192 & 0.276 & 0.404          & 0.533          & 0.449          & 0.392 \\
    GSM8K         & 0.274 & 0.332 & 0.427          & 0.455          & 0.505          & {0.382} \\
    GPQA-D        & 0.203 & 0.240 & {0.346}        & 0.405          & 0.498          & 0.400 \\
    HumanEval     & 0.205 & 0.374 & 0.428          & 0.580          & {0.478}        & 0.421 \\
    LiveCodeBench & 0.222 & 0.260 & {0.530}        & 0.511          & 0.482          & 0.425 \\
    \bottomrule
  \end{tabular}
  \end{adjustbox}
\end{table}

\paragraph{Confidence of Trace }
\begin{figure}[t]
    \centering
    \begin{minipage}[t]{0.60\linewidth}
        \centering
        \includegraphics[height=4.2cm,width=\linewidth,keepaspectratio]{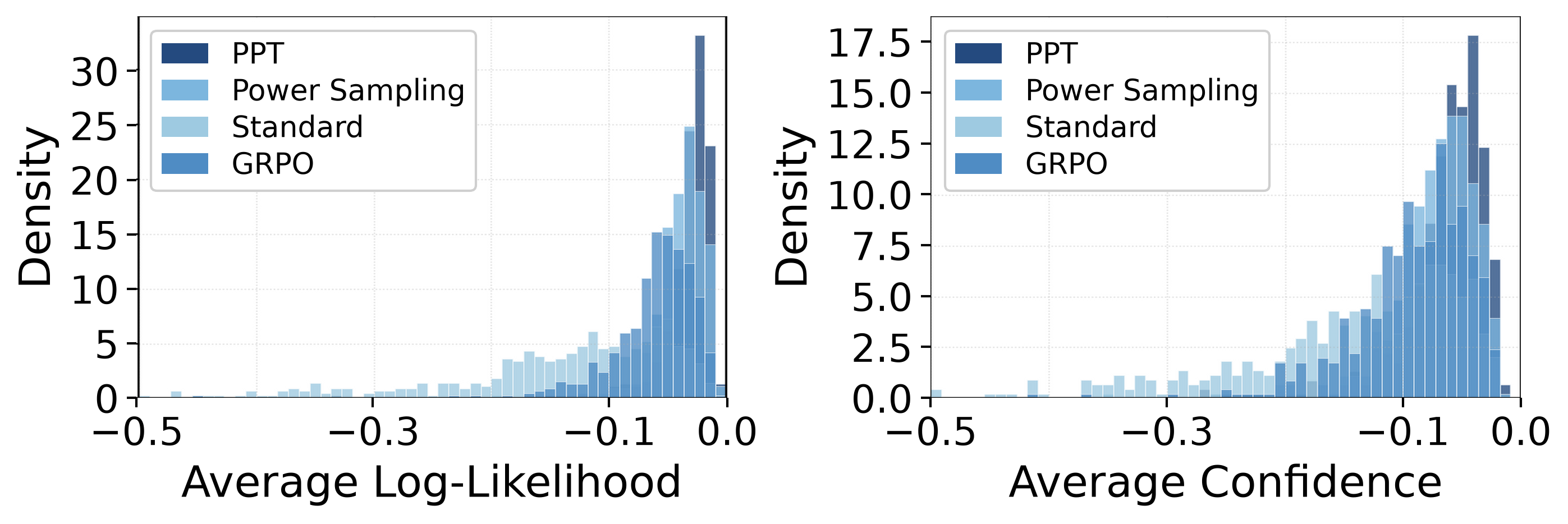}
        \caption{Confidence comparison across methods on MATH500 using Qwen3-8B; values closer to zero indicate more plausible traces.}
        \label{fig:confidence}
    \end{minipage}\hfill
    \begin{minipage}[t]{0.37\linewidth}
        \centering
        \safeincludegraphics[width=0.88\linewidth]{figures/acc_vs_mcmc.png}
        \caption{Accuracy gain over the one-step baseline for varying number of refinement steps.}
        \label{fig:local_update_ablation}
    \end{minipage}
\end{figure}

Figure~\ref{fig:confidence} complements task accuracy by asking whether the returned traces lie in regions that the model itself regards as plausible. As illustrated, for both average log-likelihood and negative next-token entropy, \method{} concentrates more of the distribution near zero and suppresses the diffuse low-score tail seen most clearly under standard decoding. This is the behavior expected from sequence-level sharpening: trajectories discovered across the ladder can move to the output rung when they are compatible with the sharper target. These diagnostics measure the model's internal preference rather than calibration, so they support the intended concentration mechanism but do not by themselves establish correctness.

\subsection{Ablation Studies}
\label{apx:ablation}

\paragraph{Impact of MCMC steps.}
The number of local refinement steps $N_{\mathrm{MCMC}}$ specifies the number of MH-update attempts per replica per block.
Figure~\ref{fig:local_update_ablation} demonstrates the accuracy gain over a single step, as most of the improvement arrives by $4$ steps ($+3.0$ on MATH500, $+4.1$ on HumanEval), and the best results come at $10$ steps ($+4.4$ and $+6.6$).
More refinement therefore helps overall but the gain is not monotone in $N_{\mathrm{MCMC}}$ 
and since local updates dominate runtime (Table~\ref{tab:time-memory-main}), $N_{\mathrm{MCMC}}$ is the main knob for trading accuracy against cost.

\paragraph{Compute-matched comparison}
\label{apx:compute-matched}
In this section, we consider two complementary compute-matched controls: a single chain given more MCMC refinement steps, and multiple uncoupled chains.  

\emph{More refinement in a single chain.}
First, Table~\ref{tab:compute-matched-single-chain} compares \method{} with $K>1$ replicas against a single chain ($K=1$). We increase the number of MCMC steps in the single-chain run until its total generated-token budget matches or even exceeds that of the multi-replica run. Both methods return one trace, so this comparison tests whether the same token budget is better spent on additional refinement of one chain or on interacting replicas. \method{} achieves consistently higher accuracy across all eight comparisons, with the largest gains on AIME.

\begin{table}[!htb]
  \centering
  \caption{Compute-matched single-trace comparison. Entries report Accuracy (Tokens$\times10^3$), with accuracy in percent and decode tokens per problem in parentheses, including refinement proposals. The $K=1$ control uses additional MCMC steps to match the decode-token budget of \method{} ($K>1$). \textbf{Bold} marks the higher accuracy for each benchmark and model.}
  \label{tab:compute-matched-single-chain}
  \setlength{\tabcolsep}{5pt}
  \begin{adjustbox}{max width=\linewidth}
  \begin{tabular}{l|c|c|c|c}
    \toprule
    \textbf{Method} & \textbf{MATH 500} & \textbf{GPQA} & \textbf{AIME 24\&25} & \textbf{LCB~v5} \\
    \midrule
    \multicolumn{5}{l}{\textbf{Qwen3-4B}} \\
    \hspace{12pt} Single chain ($K=1$)
    & 84.8 (112) & 53.0 (212) & 72.7 (95) & 56.0 (255) \\
    \rowcolor{pptrow} \hspace{12pt} \textbf{\method{} (Ours)} ($K>1$)
    & \textbf{86.0} (83) & \textbf{53.6} (146) & \textbf{78.3} (87) & \textbf{57.9} (233) \\
    \midrule
    \multicolumn{5}{l}{\textbf{Qwen3-8B}} \\
    \hspace{12pt} Single chain ($K=1$)
    & 85.6 (131) & 59.1 (363) & 75.3 (215) & 54.9 (213) \\
    \rowcolor{pptrow} \hspace{12pt} \textbf{\method{} (Ours)} ($K>1$)
    & \textbf{88.0} (89) & \textbf{60.1} (349) & \textbf{78.3} (151) & \textbf{58.4} (210) \\
    \bottomrule
  \end{tabular}
  \end{adjustbox}
\end{table}

\emph{Multiple traces.}
Subsequently, Table~\ref{tab:compute-matched-app} compares \method{} against two multi-trace baselines, each given the same decode-token budget as the coupled run.
The \emph{uncoupled ladder} keeps the replica count, powers $\{\alpha_k\}$, horizon, block size, and refinement schedule of \method{} but disables all swap attempts, so it isolates the effect of exchanging records during generation.
\emph{Best-of-$N$} (BoN) draws independent samples from the model until the budget is used, so it represents the standard way of spending extra inference compute across many traces.
For our metrics, we use:
\emph{Pass@1} returns one final record;
\emph{Vote} takes the majority over final answers, breaking ties by cumulative base-model log-likelihood.

For Pass@1, \method{} returns its coldest-rung ($k^\star$) trace and \emph{Best-of-$N$} its highest-log-likelihood trace; the \emph{uncoupled ladder} is reported under both rules for a like-for-like comparison.
The single-trace readout gives a clear comparison of sampling quality, and there \method{} is the strongest method in every setting.
The records returned by \method consistently outperform both the traces from the uncoupled ladder and from BoN at the same compute.
Compared with the uncoupled ladder, exchange improves the single-trace and vote accuracy in every model--benchmark setting, and on AIME 24\&25 it raises all three settings for both models, by a considerable margin.
This shows that the gains come from how exchange reshapes each chain's trajectory, not from spending more tokens.

Taken together, these controls show that the benefit of \method{} holds whether compute is matched through deeper refinement of one chain, several chains without exchange, or independent base-model samples.

\begin{table*}[!t]
\centering
\caption{Compute-matched accuracy (\%) of \method{}, the uncoupled ladder (no swaps), and Best-of-$N$ under the same decode-token budget. For the uncoupled ladder, the two Pass@1 values correspond to returning the highest-likelihood trace (Likelihood Rank) and the trace at rung $k^\star$ (Coldest Rung). \textbf{Bold} marks the best result within each column and model group.}
\label{tab:compute-matched-app}
\setlength{\tabcolsep}{3pt}
\begin{adjustbox}{max width=\textwidth}
\begin{tabular}{l|cc|cc|cc}
\toprule
\multirow{2}{*}{\textbf{Method}} &
\multicolumn{2}{c|}{\textbf{MATH 500}} &
\multicolumn{2}{c|}{\textbf{GPQA}} &
\multicolumn{2}{c}{\textbf{AIME 24\&25}} \\
\cmidrule(lr){2-3} \cmidrule(lr){4-5} \cmidrule(lr){6-7}
& Pass@1 & Vote
& Pass@1 & Vote
& Pass@1 & Vote \\
\midrule
\multicolumn{7}{l}{\textbf{Qwen3-4B}} \\
\hspace{12pt} \begin{tabular}[b]{@{}l@{}}Uncoupled ladder\\(Likelihood\,/\,Coldest Rung)\end{tabular}
& 84.2\,/\,83.6 & 87.6
& 52.3\,/\,52.9 & 48.5
& 75.7\,/\,74.3 & 78.3 \\
\hspace{12pt} Best-of-$N$
& 84.7 & \textbf{89.2}
& 52.4 & 54.7
& 71.3 & 77.7 \\
\rowcolor{pptrow} \hspace{12pt} \textbf{\method{} (Ours)}
& \textbf{86.0} & 88.8
& \textbf{53.6} & \textbf{55.1}
& \textbf{78.3} & \textbf{80.0} \\
\midrule
\multicolumn{7}{l}{\textbf{Qwen3-8B}} \\
\hspace{12pt} \begin{tabular}[b]{@{}l@{}}Uncoupled ladder\\(Likelihood\,/\,Coldest Rung)\end{tabular}
& 86.0\,/\,84.9 & 87.4
& 57.6\,/\,57.6 & 58.9
& 76.7\,/\,76.7 & 80.0 \\
\hspace{12pt} Best-of-$N$
& 84.6 & \textbf{90.3}
& 56.5 & 60.1
& 72.8 & 78.6 \\
\rowcolor{pptrow} \hspace{12pt} \textbf{\method{} (Ours)}
& \textbf{88.0} & 90.0
& \textbf{60.1} & \textbf{61.6}
& \textbf{78.3} & \textbf{81.7} \\
\bottomrule
\end{tabular}
\end{adjustbox}
\end{table*}

\paragraph{Fixed Horizon vs One-Way Truncation}
Table~\ref{tab:cold-replica-results} separates the effect of the fixed-horizon correction from the effect of replica exchange.
 With exchange disabled, moving to a common fixed-horizon state space improves accuracy for both models. Retaining reversible suffix moves lets refinement reconsider where a response should end, rather than making an early EOS an irreversible change of state space. 
  Enabling exchange further improves accuracy. The strongest results consequently come from two complementary ingredients---a valid fixed-horizon kernel that can revise 
  trajectories in both directions, and communication that transports useful trajectories across sharpening levels.
\\
The two kernels treat termination asymmetrically. Under the
truncating kernel, an accepted early EOS permanently shrinks the proposal
budget of every later local move, while a record that has not terminated
can only be shortened by a proposal that itself terminates. Consequently, the terminal
position is effectively frozen after a few updates. 
Conversely, under the
fixed-horizon kernel, a restart at or before the terminal position
regenerates the suffix through $T_m$ in both directions, so refinement can
move the EOS earlier or later and can convert a right-censored record into
a terminated one. 

Consistent with this, moving from the truncating to the
fixed-horizon kernel at $K=1$ lowers the hit-capacity rate from $9.2\%$ to
$8.0\%$ on Qwen3-4B and from $12.2\%$ to $9.5\%$ on Qwen3-8B, so more responses
terminate within budget, and accuracy rises at the same time. Enabling
exchange on top of the fixed-horizon kernel leaves the hit-capacity rate
essentially unchanged ($7.6\%$ and $9.6\%$) while improving accuracy further,
which is the expected signature of swaps acting on \emph{which} trajectories
reach the output rung rather than on termination behavior. The correction
is therefore practically useful as well as theoretically necessary, and the
two ingredients of PPT contribute through distinct mechanisms.

\begin{table}[!t]
  \centering
  \caption{Pass@1 accuracy (\%) and hit-capacity rate (\%, share of returned responses that exhaust the completion budget without emitting EOS) on MATH500 for the truncating single-chain sampler (Power Sampling), a fixed-horizon single chain, and \method{} (fixed horizon, $K>1$ replicas with swaps). \textbf{Bold} marks the best accuracy for each model.}
  \label{tab:cold-replica-results}
  \setlength{\tabcolsep}{5pt}
  \begin{adjustbox}{max width=\linewidth}
  \begin{tabular}{l|cc|cc}
    \toprule
    \multirow{2}{*}{\textbf{Method}} & \multicolumn{2}{c|}{\textbf{Qwen3-4B}} & \multicolumn{2}{c}{\textbf{Qwen3-8B}} \\
    & Acc.\ $\uparrow$ & Hit cap.\ $\downarrow$ & Acc.\ $\uparrow$ & Hit cap.\ $\downarrow$ \\
    \midrule
    Truncating, $K=1$ (Power Sampling) & 83.5 & 9.2 & 82.2 & 12.2 \\
    Fixed-horizon, $K=1$ & 84.5 & 8.0 & 85.3 & 9.5 \\
    \rowcolor{pptrow} \textbf{\method{} (Ours)} (fixed-horizon, $K>1$) & \textbf{86.0} & 7.6 & \textbf{88.0} & 9.6 \\
    \bottomrule
  \end{tabular}
  \end{adjustbox}
\end{table}

\paragraph{Communication Schedule: DEO versus Ordered Adjacent Sweeps}
\label{apx:deo-adj}
Appendix~\ref{sec:chain-communication} compares the ordered adjacent
sweep (ADJ) with deterministic even--odd communication (DEO). 
Table~\ref{tab:deo-adj} presents an empirical comparison between the schedules showing
that, for the same finite iteration budget, 
ADJ consistently achieves higher accuracy than DEO across all evaluated models and benchmarks, 
suggesting that its greater exploration per iteration translates into improved predictive performance.

\begin{table}[t]
\centering
\caption{Communication-schedule ablation: single-sample accuracy (\%) of \method{} with DEO versus the ordered adjacent sweep (ADJ, our default), under otherwise identical configurations. \textbf{Bold} marks the better schedule.}
\label{tab:deo-adj}
\small
\begin{tabular}{l|ccccc}
\toprule
\textbf{Schedule} &
\textbf{MATH500} &
\textbf{GPQA} &
\textbf{HumanEval} &
\textbf{AIME 24\&25} &
\textbf{LCB v5} \\
\midrule
\multicolumn{6}{l}{\textbf{Qwen 2.5 MATH}} \\
\hspace{12pt} \method{} with DEO
& 77.8 & 32.3 & 57.6 & 15.7 & 11.3 \\
\rowcolor{pptrow} \hspace{12pt} \method{} with ADJ (default)
& \textbf{79.6} & \textbf{37.6} & \textbf{60.6} & \textbf{17.7} & \textbf{12.5} \\
\midrule
\multicolumn{6}{l}{\textbf{Qwen3-4B}} \\
\hspace{12pt} \method{} with DEO
& 84.7 & 53.0 & 90.2 & 73.3 & 53.8 \\
\rowcolor{pptrow} \hspace{12pt} \method{} with ADJ (default)
& \textbf{86.0} & \textbf{53.6} & \textbf{93.9} & \textbf{78.3} & \textbf{57.9} \\
\midrule
\multicolumn{6}{l}{\textbf{Qwen3-8B}} \\
\hspace{12pt} \method{} with DEO
& 85.9 & 57.1 & 90.2 & 75.7 & 55.3 \\
\rowcolor{pptrow} \hspace{12pt} \method{} with ADJ (default)
& \textbf{88.0} & \textbf{60.1} & \textbf{94.9} & \textbf{78.3} & \textbf{58.4} \\
\bottomrule
\end{tabular}
\end{table}

\end{document}